\documentclass{article} 
\usepackage{iclr2024_conference,times}

\usepackage{amsmath,amsfonts,bm}

\def\Figref#1{Figure~\ref{#1}}

\def\1{\bm{1}}

\DeclareMathAlphabet{\mathsfit}{\encodingdefault}{\sfdefault}{m}{sl}
\SetMathAlphabet{\mathsfit}{bold}{\encodingdefault}{\sfdefault}{bx}{n}

\newcommand{\E}{\mathbb{E}}

\DeclareMathOperator*{\argmax}{arg\,max}
\DeclareMathOperator*{\argmin}{arg\,min}

\DeclareMathOperator{\Tr}{Tr}

\usepackage[utf8]{inputenc} 
\usepackage[T1]{fontenc}    
\usepackage{hyperref}       
\usepackage{url}            
\usepackage{booktabs}       
\usepackage{amsfonts}       
\usepackage{nicefrac}       
\usepackage{microtype}      
\usepackage{xcolor}         
\usepackage{wrapfig}
\usepackage{makecell}

\usepackage{amsmath}
\usepackage{amssymb}
\usepackage{mathtools}
\usepackage{amsthm}
\usepackage{multirow}
\usepackage{graphicx}
\newbox{\bigpicturebox}

\usepackage{multirow}
\usepackage{adjustbox} 
\usepackage{pdflscape}
\usepackage{afterpage}
\usepackage{tikz}
\usepackage{graphics}
\usepackage{wrapfig}

\usepackage{lipsum,booktabs}

\usepackage[capitalize,noabbrev]{cleveref}

\theoremstyle{plain}
\newtheorem{theorem}{Theorem}[section]
\newtheorem*{theorem*}{Theorem}
\newtheorem{proposition}[theorem]{Proposition}
\newtheorem*{proposition*}{Proposition}

\theoremstyle{definition}
\newtheorem{definition}[theorem]{Definition}

\theoremstyle{remark}

\usepackage{mathtools}
\usepackage[textsize=tiny]{todonotes}

\usepackage{mathtools}
\usepackage{algorithm} 
\usepackage[noend]{algorithmic}
\usepackage{amsmath}

\usepackage{titlesec}
\titlespacing*{\section}{0pt}{0.15\baselineskip}{0.15\baselineskip}
\titlespacing*{\subsection}{0pt}{0.15\baselineskip}{0.15\baselineskip}
\titlespacing*{\subsubsection}{0pt}{0.1\baselineskip}{0.1\baselineskip}

\newcommand{\safa}[1]{\color{brown}{#1}}

\newlength\myindent
\usepackage{comment}
\usepackage{xspace}
\usepackage{subcaption}
\newcommand{\tabref}[1]{Table~\ref{#1}}

\newcommand{\cH}{\mathcal{H}}
\newcommand{\cS}{\mathcal{S}}
\newcommand{\cD}{\mathcal{D}}
\newcommand{\cL}{\mathcal{\vect L}}

\newcommand{\cO}{\mathcal{O}}
\newcommand{\cA}{\mathcal{\vect A}}

\newcommand{\cN}{\mathcal{N}}

\newcommand{\vect}[1]{\boldsymbol{#1}}

\def\eg{\emph{e.g.}, } 
\def\ie{\emph{i.e.}, }

\def\etal{\emph{et. al} }

\newcommand{\STAC}{S$^2$AC}

\newcommand{\proofsketch}{\vspace*{-1ex}\noindent {\textit{Proof Sketch:} }}

\def\eg{\emph{e.g.}, } 
\def\ie{\emph{i.e.}, }

\def\etal{\emph{et. al} }

\title{\STAC: Energy-Based Reinforcement Learning with Stein Soft Actor Critic}

\author{
    \begin{tabular}{l}Safa Messaoud$^{1\dagger}$, Billel Mokeddem$^{1*}$,
    Zhenghai Xue$^{2*}$,
    Linsey Pang$^3$,
    Bo An$^{4,2}$,\vspace{1mm}\\
    Haipeng Chen$^{5\dagger}$,
    Sanjay Chawla$^{1\dagger}$
    \end{tabular}\vspace{1mm}\\ 
    $^1$Qatar Computing Research Institute, Hamad Bin Khalifa University,  $^2$School of Computer Science and\\ Engineering, Nanyang Technological University, $^3$SalesForce, $^4$Skywork AI, $^5$Data Science, William \& Mary\\
    \small{\texttt{\{smessaoud,bmokeddem,schawla\}@hbku.edu.qa, zhenghai001@e.ntu.edu.sg}}\\
    \small{\texttt{panglinsey@gmail.com, boan@ntu.edu.sg, hchen23@wm.edu}}\\
    \small{$^*$ Equal contribution} \ \ 
    \small{$^{\dagger}$ Corresponding authors}
}

\iclrfinalcopy 
\begin{document}

\maketitle
\vspace{-6mm}
\begin{abstract}
\vspace{-3mm}
Learning expressive stochastic policies instead of deterministic ones has been proposed to achieve better stability, sample complexity, and robustness. Notably, in Maximum Entropy Reinforcement Learning (MaxEnt RL), the policy is modeled as an expressive Energy-Based Model (EBM) over the Q-values. However, this formulation requires the estimation of the entropy of such EBMs, which is an open problem. To address this, previous MaxEnt RL methods either implicitly estimate the entropy, resulting in high computational complexity and variance (SQL), or follow a variational inference procedure that fits simplified actor distributions (\eg Gaussian) for tractability (SAC). We propose \underline{S}tein \underline{S}oft \underline{A}ctor-\underline{C}ritic (\STAC), a MaxEnt RL algorithm that learns expressive policies without compromising efficiency. Specifically, \STAC\ uses parameterized Stein Variational Gradient Descent (SVGD) as the underlying policy. We derive a closed-form expression of the entropy of such policies. Our formula is computationally efficient and only depends on first-order derivatives and vector products. Empirical results show that \STAC\ yields more optimal solutions to the MaxEnt objective than SQL and SAC in the multi-goal environment, and outperforms SAC and SQL on the MuJoCo benchmark. Our code is available at: \url{https://github.com/SafaMessaoud/S2AC-Energy-Based-RL-with-Stein-Soft-Actor-Critic}
\end{abstract}

\vspace{-2mm}\section{Introduction}
\label{sec:introduction}

\begin{wrapfigure}{r}{0.4\textwidth} 
\centering
\includegraphics[width=0.4\textwidth]{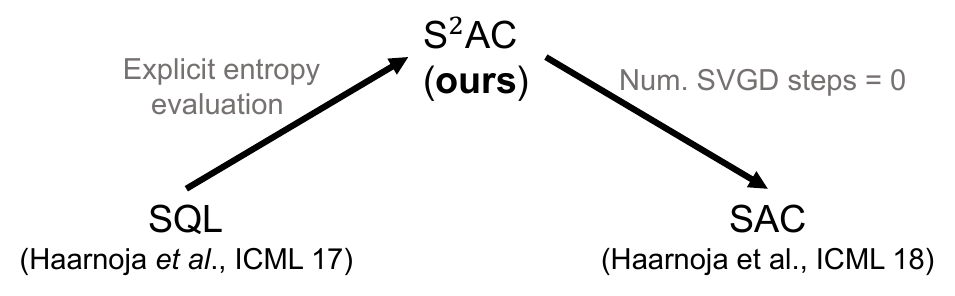}
\caption{Comparing \STAC\ to SQL and SAC. \STAC\ with a parameterized policy is reduced to SAC if the number of SVGD steps is 0. SQL becomes equivalent to \STAC\ if the entropy is evaluated explicitly with our derived formula.}\vspace{-3mm}
\label{fig:teaser_new}
\end{wrapfigure}

MaxEnt RL \citep{todorov2006linearly,ziebart2010modeling,haarnoja2017reinforcement,kappen2005path,toussaint2009robot,theodorou2010generalized,abdolmaleki2018maximum,haarnoja2018soft,vieillard2020munchausen} has been proposed to address challenges hampering the deployment of RL to real-world applications, including stability, sample efficiency ~\citep{gu2017q}, and robustness \citep{eysenbach2022maximum}. Instead of learning a deterministic policy, as in classical RL \citep{sutton1999policy,schulman2017proximal,silver2014deterministic, lillicrap2015continuous}, MaxEnt RL learns a stochastic policy that captures the intricacies of the action space. This enables better exploration during training and eventually better robustness to environmental perturbations at test time, \ie the agent learns multimodal action space distributions which enables picking the next best action in case a perturbation prevents the execution of the optimal one. To achieve this, MaxEnt RL models the policy using the expressive family of EBMs \citep{lecun2006tutorial}. This translates into learning policies that maximize the sum of expected future reward and expected future entropy. However, estimating the entropy of such complex distributions remains an open problem. 
\vspace{-1mm}

To address this, existing approaches either use tricks to go around the entropy computation or make limiting assumptions on the policy. This results in either poor scalability or convergence to suboptimal solutions. For example, SQL \citep{haarnoja2017reinforcement} implicitly incorporates entropy in the Q-function computation. This requires using importance sampling, which results in high variability and hence poor training stability and limited scalability to high dimensional action spaces. SAC \citep{haarnoja2018soft}, on the other hand, follows a variational inference procedure by fitting a Gaussian distribution to the EBM policy. This enables a closed-form evaluation of the entropy but results in a suboptimal solution. For instance, SAC fails in environments characterized by multimodal action distributions. Similar to SAC, IAPO \citep{marino2021iterative} models the policy as a uni-modal Gaussian. Instead of optimizing a MaxEnt objective, it achieves multimodal policies by learning a collection of parameter estimates (mean, variance) through different initializations for different policies. To improve the expressiveness of SAC, SSPG \citep{cetin2022policy} and SAC-NF \citep{mazoure2020leveraging} model the policy as a Markov chain with Gaussian transition probabilities and as a normalizing flow \citep{rezende2015variational}, respectively. However, due to training stability issues, the reported results in \cite{cetin2022policy} show that though both models learn multi-modal policies, they fail to maximize the expected future entropy in positive rewards setups.
\vspace{-1mm}

We propose a new algorithm, \STAC, that yields a more optimal solution to the MaxEnt RL objective. To achieve \textit{expressivity}, \STAC\ models the policy as a Stein Variational Gradient Descent (SVGD) \citep{liu2017stein} sampler from an EBM over Q-values (target distribution). SVGD proceeds by first sampling a set of particles from an initial distribution, and then iteratively transforming these particles via a sequence of updates to fit the target distribution. To compute a \textit{closed-form estimate of the entropy} of such policies, we use the change-of-variable formula for pdfs \citep{devore2012modern}. We prove that this is only possible due to the invertibility of the SVGD update rule, which does not necessarily hold for other popular samplers (\eg Langevin Dynamics \citep{welling2011bayesian}). While normalizing flow models \citep{rezende2015variational} are also invertible, SVGD-based policy is more expressive as it encodes the inductive bias about the unnormalized density and incorporates a dispersion term to encourage multi-modality, whereas normalizing flows encode a restrictive class of invertible transformations (with easy-to-estimate Jacobian determinants). Moreover, our formula is computationally efficient and only requires evaluating first-order derivatives and vector products. To improve \textit{scalability}, we model the initial distribution of the SVGD sampler as an isotropic Gaussian and learn its parameters, \ie mean and standard deviation, end-to-end. We show that this results in faster convergence to the target distribution, \ie fewer SVGD steps. Intuitively, the initial distribution learns to contour the high-density region of the target distribution while the SVGD updates result in better and faster convergence to the modes within that region. Hence, our approach is as parameter efficient as SAC, since the SVGD updates do not introduce additional trainable parameters.\vspace{-1mm}

Note that \STAC\ can be reduced to SAC when the number of SVGD steps is zero. Also, SQL becomes equivalent to \STAC\ if the entropy is computed explicitly using our formula (the policy in SQL is an amortized SVGD sampler). Beyond RL, the backbone of \STAC\ is a new variational inference algorithm with a more expressive and scalable distribution characterized by a closed-form entropy estimate. We believe that this variational distribution can have a wider range of exciting applications. \vspace{-1mm}

\begin{figure*}
    \centering
\includegraphics[width=0.8\linewidth]{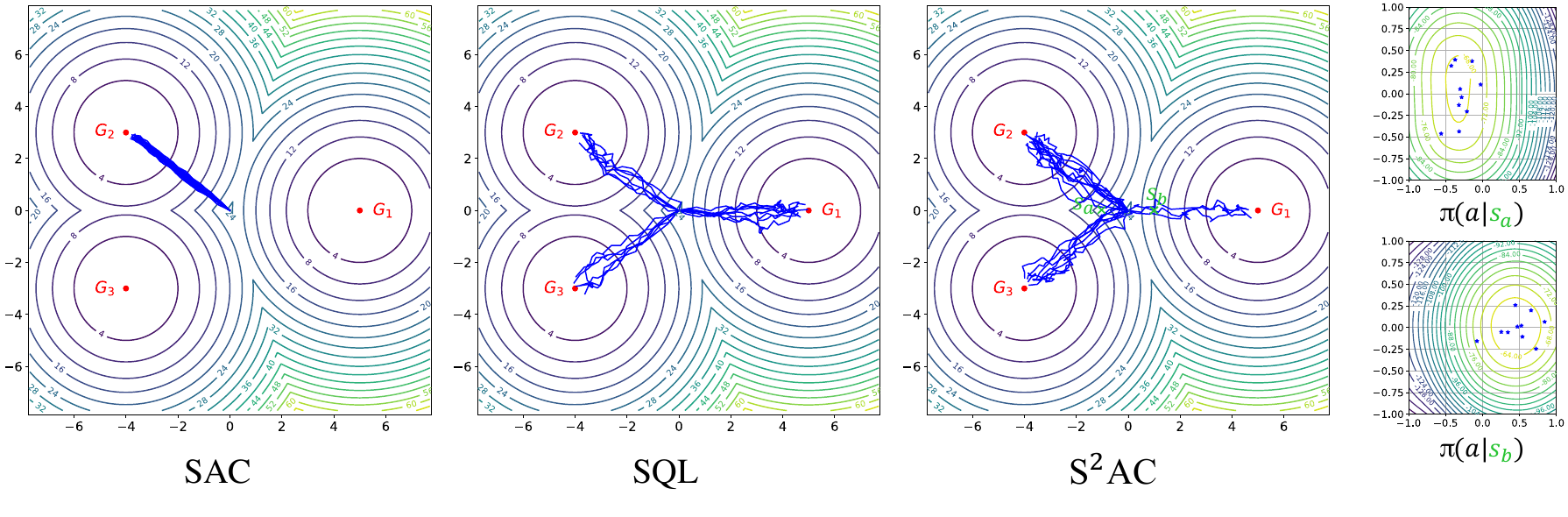}
\vspace{-5mm}
\caption{\STAC\ learns a more optimal solution to the MaxEnt RL objective than SAC and SQL. We design a multigoal environment where an agent starts from the center of the 2-d map and tries to reach one of the three goals ($G_1$, $G_2$, and $G_3$). The maximum expected future reward (level curves) is the same for all the goals but the expected future entropy is different (higher on the path to $G_2/G_3$): the action distribution $\pi(a|s)$ is bi-modal on the path to the left ($G_2$ and $G_3$) and unimodal to the right ($G_1$). Hence, we expect the optimal policy for the MaxEnt RL objective to assign more weights to $G_2$ and $G_3$. We visualize trajectories (in blue) sampled from the policies learned using SAC, SQL, and \STAC. SAC quickly commits to a single mode due to its actor being tied to a Gaussian policy. Though SQL also recovers the three modes, the trajectories are evenly distributed. \STAC\ recovers all the modes and approaches the left two goals more frequently. This indicates that it successfully maximizes not only the expected future reward but also the expected future entropy.}
\vspace{-8mm}
\label{fig:multi-goal-teaser}
\end{figure*}

We conduct extensive empirical evaluations of \STAC\ from three aspects. We start with a sanity check on the merit of our derived SVGD-based entropy estimate on target distributions with known entropy values (\eg Gaussian) or log-likelihoods (\eg Gaussian Mixture Models) and assess its sensitivity to different SVGD parameters (kernel, initial distribution, number of steps and number of particles). We observe that its performance depends on the choice of the kernel and is robust to variations of the remaining parameters. In particular, we find out that the kernel should be chosen to guarantee inter-dependencies between the particles, which turns out to be essential for invertibility. Next, we assess the performance of \STAC\ on a multi-goal environment \citep{haarnoja2017reinforcement} where different goals are associated with the same positive (maximum) expected future reward but different (maximum) expected future entropy. We show that \STAC\ learns multimodal policies and effectively maximizes the entropy, leading to better robustness to obstacles placed at test time. Finally, we test \STAC\ on the MuJoCo benchmark \citep{duan2016benchmarking}. \STAC\ yields better performances than the baselines on four out of the five environments. Moreover, \STAC\ shows higher sample efficiency as it tends to converge with fewer training steps. These results were obtained from running SVGD for only three steps, which results in a small overhead compared to SAC during training. Furthermore, to maximize the run-time efficiency during testing, we train an amortized SVGD version of the policy to mimic the SVGD-based policy. Hence, this reduces inference to a forward pass through the policy network without compromising the performance.

\vspace{-1mm}\section{Preliminaries}\label{sec:preliminary}
\vspace{-1mm}
\subsection{Samplers for Energy-based Models}\label{sec:ebm_samplers}
\vspace{-1mm}
In this work, we study three representative methods for sampling from EBMs: (1) Stochastic Gradient Langevin Dynamics (SGLD) \& Deterministic Langevin Dynamics (DLD)~\citep{welling2011bayesian}, (2) Hamiltonian Monte Carlo (HMC)~\citep{neal2011mcmc}, and (3) Stein Variational Gradient Descent (SVGD)~\citep{liu2016stein}. We review SVGD here since it is the sampler we eventually use in \STAC, and leave the rest to Appendix~\ref{appendix:additional_samplers}. 
SVGD is a particle-based Bayesian inference algorithm. Compared to SGLD and HMC which have a single particle in their dynamics, SVGD operates on a set of particles. Specifically, SVGD samples a set of $m$  particles $\{a_j\}_{j=1}^m$ from an initial distribution $q^{0}$ which it then transforms through a sequence of updates to fit the target distribution. Formally, at every iteration $l$, SVGD applies a form of functional gradient descent $\Delta f$ that minimizes the KL-divergence between the target distribution $p$ and the proposal distribution $q^{l}$ induced by the particles, \ie the update rule for the $i^{\text{th}}$ particles is: $a^{l+1}_{i} = a^{l}_{i} + \epsilon \Delta f(a_{i}^{l})$ with 
\vspace{-1mm}
\begin{equation}\label{eq:svgd_update}
\Delta f(a_{i}^{l}) = \mathbb{E}_{a_{j}^l\sim q^l}\big[ k(a^{l}_{i},a^{l}_{j})\nabla_{a^{l}_j}\log p(a^{l}_{j}) + \nabla_{a^{l}_{j}}k(a^{l}_i,a^{l}_j)\big] .
\end{equation} 
Here, $\epsilon$ is the step size and $k(\cdot,\cdot)$ is the kernel function, \eg the RBF kernel: 
$k(a_i,a_j) = \exp(||a_i -a_j||^2/2 \sigma^2)$. 
The first term within the gradient drives the particles toward the high probability regions of $p$, while the second term serves as a repulsive force to encourage dispersion. 
\subsection{Maximum-Entropy RL}
\vspace{-2mm}
We consider an infinite horizon Markov Decision Process (MDP) defined by a tuple $(\cS,\cA,p,r)$, where $\cS$ is the state space, $\cA$ is the action space and $p: \cS \times \cA \times \cS \rightarrow [0,\infty]$ is the state transition probability modeling the density of the next state $s_{t+1} \in \cS$ given the current state $s_t \in \mathcal{S}$ and action $a_t \in \mathcal{A}$. Additionally, we assume that the environment emits a bounded reward function $r \in [r_{\text{min}}, r_{\text{max}}]$ at every iteration. We use $\rho_{\pi}(s_t)$ and $\rho_{\pi}(s_t, a_t)$ to denote the state and state-action marginals of the trajectory distribution induced by a policy $\pi (a_t | s_t)$. We consider the setup of continuous action spaces \cite{lazaric2007reinforcement,lee2018deep,zhou2023single}. MaxEnt RL~\citep{todorov2006linearly,ziebart2010modeling,rawlik2012stochastic} learns a policy $\pi^{*}(a_t|s_t)$,  that instead of maximizing the expected future reward, maximizes the sum of the expected future reward and entropy:
\begin{equation}\label{eq:max_entr_obj}
\pi^{*}= \argmax\nolimits_{\pi}  \sum\nolimits_t \gamma^t \mathbb{E}_{(s_t,a_t)\sim \rho_{\pi}} \big[ r(s_t,a_t) + \alpha \mathcal{H}(\pi(\cdot|s_t))\big],
\end{equation}
where $\alpha$ is a temperature parameter controlling the stochasticity of the policy and $\mathcal{H}(\pi(\cdot|s_t))$ is the entropy of the policy at state $s_t$. The conventional RL objective can be recovered for $\alpha=0$. Note that the MaxEnt RL objective above is equivalent to approximating the policy, modeled as an EBM over Q-values, by a variational distribution $\pi(a_t|s_t)$ (see proof of equivalence in Appendix~\ref{appendix:MaxEnt_sol}), \ie 
\begin{equation}\label{eq:ebm}
\pi^{*} = \argmin\nolimits_{\pi} \sum\nolimits_t \mathbb{E}_{s_t \sim \rho_\pi}\big[ D_{K\!L}\big(\pi(\cdot|s_t) \| \exp\!{(  Q(s_t,\cdot)/\alpha )}/Z \big)\big],
\end{equation}
where $D_{KL}$ is the KL-divergence and $Z$ is the normalizing constant. We now review two landmark MaxEnt RL algorithms: SAC \citep{haarnoja2018soft} and SQL \citep{haarnoja2017reinforcement}.

\textbf{SAC} is an actor-critic algorithm that alternates between policy evaluation, \ie evaluating the Q-values for a policy $\pi_{\theta}(a_t|s_t)$:
\begin{equation}\label{eq:sac_pe}
Q_{\phi}(s_t,a_t) \leftarrow r(s_t,a_t)+ \gamma\mathop{\mathbb{E}}\nolimits_{s_{t+1},a_{t+1}\sim \rho_{\pi_{\theta}}}\big[ Q_{\phi}(s_{t+1},a_{t+1}) + \alpha  \cH(\pi_{\theta}(\cdot|s_{t+1}))  \big]
\end{equation} 
and policy improvement, \ie using the updated Q-values to compute a better policy:
\begin{equation}\label{eq:sac_pi}
\pi_{\theta}= \argmax\nolimits_{\theta} \sum\nolimits_t \mathop{\mathbb{E}_{s_t,a_t \sim \rho_{\pi_{\theta}} } } \big[ Q_{\phi}(a_t,s_t) + \alpha \cH(\pi_{\theta}(\cdot|s_t) )  \big].
\end{equation}
SAC models $\pi_{\theta}$ as an isotropic Gaussian, \ie $\pi_{\theta}(\cdot |s ) = \cN(\mu_{\theta}, \sigma_{\theta}I)$.  While this enables computing a closed-form expression of the entropy, it incurs an over-simplification of the true action distribution, and thus cannot represent complex distributions, \eg multimodal distributions.

\textbf{SQL} goes around the entropy computation, by defining a soft version of the value function $V_{\phi}=\alpha \log \big(  \int_{\cA} \exp\big( \frac{1}{\alpha} Q_{\phi}(s_t,a')  \big) da' \big)$. This enables expressing the Q-value (Eq~\eqref{eq:sac_pe}) independently from the entropy, \ie $Q_{\phi}(s_t,a_t)=r(s_t,a_t)+\gamma \E_{s_{t+1}\sim p}[V_{\phi}(s_{t+1})]$. Hence, SQL follows a soft value iteration which alternates between the updates of the ``soft'' versions of $Q$ and value functions:
\begin{eqnarray}
\label{eq:sql_q}&Q_{\phi}(s_t,a_t) \leftarrow r(s_t,a_t)+\gamma \E_{s_{t+1}\sim p}[V_{\phi}(s_{t+1})], \ \forall (s_t, a_t)  \\
\label{eq:sql_v}&V_{\phi}(s_t) \leftarrow  \alpha \log \big(  \int_{\cA} \exp\big( \frac{1}{\alpha} Q_{\phi}(s_t,a')  \big) da' \big), \ \forall s_t.
\end{eqnarray} 
Once the $Q_{\phi}$ and $V_{\phi}$ functions converge, SQL uses amortized SVGD~\cite{wang2016learning} to learn a stochastic sampling network $f_{\theta}( \xi, s_t )$ that maps noise samples $\xi$ into the action samples from the EBM policy distribution $\pi^{*}(a_t|s_t)= \exp \big( \frac{1}{\alpha} ( Q^{*}(s_t,a_t) - V^{*}(s_t))\big)$. The parameters $\theta$ are obtained by minimizing the loss $ J_{\theta}(s_t) = D_{K\!L} \big( \pi_{\theta}(\cdot|s_t)|| \exp\!{ \big(\frac{1}{\alpha} (Q^{*}_{\phi}(s_t,\cdot)-V^{*}_{\phi}(s_t)}) \big)$ with respect to $\theta$. Here, $\pi_{\theta}$ denotes the policy induced by $f_{\theta}$. SVGD is designed to minimize such KL-divergence without explicitly computing $\pi_{\theta}$. In particular, SVGD provides the most greedy direction as a functional $\Delta f_{\theta}(\cdot,s_t)$ (Eq~\eqref{eq:svgd_update}) which can be used to approximate the gradient $\partial J_{\theta}/\partial a_t$. Hence, the gradient of the loss $J_\theta$ with respect to $\theta$ is: $\partial J_{\theta}(s_t)/\partial \theta \propto \E_{\xi} \big[ \Delta f_{\theta}(\xi, s_t)  \partial f_{\theta}(\xi, s_t)/\partial \theta \big]$. Note that the integral in Eq~\eqref{eq:sql_v} is approximated via importance sampling, which is known to result in high variance estimates and hence poor scalability to high dimensional action spaces. Moreover, amortized generation is usually unstable and prone to mode collapse, an issue similar to GANs. Therefore, SQL is outperformed by SAC~\cite{haarnoja2018soft} on benchmark tasks like MuJoCo.

\vspace{-1mm}\section{Approach}\label{sec:approach}
\vspace{-1mm}
We introduce \STAC, a new actor-critic MaxEnt RL algorithm that uses SVGD as the underlying actor to generate action samples from policies represented using EBMs. This choice is motivated by the expressivity of distributions that can be fitted via SVGD. Additionally, we show that we can derive a closed-form entropy estimate of the SVGD-induced distribution, thanks to the invertibility of the update rule, which does not necessarily hold for other EBM samplers. Besides, we propose a parameterized version of SVGD to enable scalability to high-dimensional action spaces and non-smooth Q-function landscapes. \STAC\ is hence capable of learning a more optimal solution to the MaxEnt RL objective (Eq~\eqref{eq:max_entr_obj}) as illustrated in Figure~\ref{fig:multi-goal-teaser}.
\subsection{Stein Soft Actor Critic}
\vspace{-2mm}
Like SAC, \STAC\ performs soft policy iteration which alternates between policy evaluation and policy improvement. The difference is that we model the actor as a \textit{parameterized sampler from an EBM}. Hence, the policy distribution corresponds to an expressive EBM as opposed to a Gaussian. 

\textbf{Critic.} The critic's parameters $\phi$ are obtained by minimizing the Bellman loss as traditionally:
\begin{equation}
\phi^{*} = \argmin\nolimits_{\phi}\E_{(s_t,a_t) \sim \rho_{\pi_{\theta}}} \left[ (Q_{\phi}(s_t,a_t) - \hat{y})^2  \right],
\label{eq:critic_update}
\end{equation} 
with the target $\hat{y}
= r_{t}(s_t,a_t) + \gamma \E_{(s_{t+1},a_{t+1}) \sim \rho_{\pi}} \left[Q_{\bar{\phi}}(s_{t+1},a_{t+1}) + \alpha\mathcal{H}(  \pi(\cdot|s_{t+1}))\right].$
Here $\bar{\phi}$ is an exponentially moving average of the value
network weights \citep{mnih2015human}. 

\textbf{Actor as an EBM sampler.} The actor is modeled as a sampler from an EBM over the Q-values. To generate a set of valid actions, the actor first samples a set of particles $\{a^{0}\}$ from an initial distribution $q^{0}$ (\eg Gaussian). These particles are then updated over several iterations $l \in [1,L]$, \ie $\{a^{l+1}\} \leftarrow \{a^{l}\} + \epsilon h(\{a^{l}\},s)$ following the sampler dynamics characterized by a transformation $h$ (\eg for SVGD, $h=\Delta f$ in Eq~\eqref{eq:svgd_update}). If $q^{0}$ is tractable and $h$ is invertible, it's possible to compute a closed-form expression of the distribution of the particles at the $l^{\text{th}}$ iteration via the change of variable formula \cite{devore2012modern}: $q^{l}( a^{l}|s ) = q^{l-1}( a^{l-1}|s )\left|\det (I + \epsilon \nabla_{a^l} h(a^{l},s)) \right|^{-1}, \forall l \in[1,L]$. In this case, the policy is represented using the particle distribution at the final step $L$ of the sampler dynamics, \ie $\pi(a|s)= q^{L}(a^L|s)$ and the entropy can be estimated by averaging $\log q^{L}(a^L|s)$ over a set of particles (Section \ref{sec:entropy}). We study the invertibility of popular EBM samplers in Section \ref{sec:invertible_policies}. 
\begin{wrapfigure}{r}{0.5\textwidth} 
    \centering
    \includegraphics[width=\linewidth]{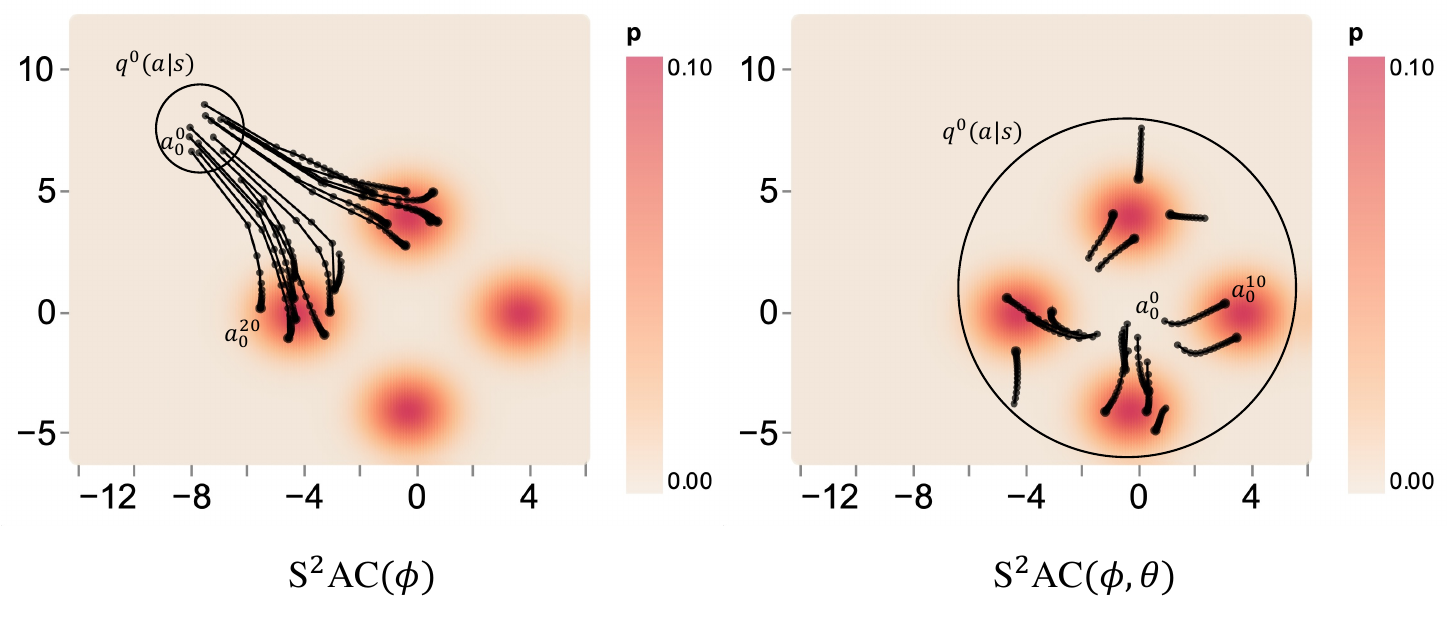}
    \vspace{-7mm}
    \caption{\STAC($\phi, \theta$) achieves faster convergence to the target distribution (in orange) than \STAC($\phi$) by parameterizing the initial distribution $\mathcal{N}(\mu_{\theta},\sigma_{\theta})$ of the SVGD sampler.}
    \label{fig:actor}
    \vspace{-3.5mm}
\end{wrapfigure}

 
\vspace{-4mm}
\textbf{Parameterized initialization.} To reduce the number of steps required to converge to the target distribution (hence reducing computation cost), we further propose modeling the initial distribution as a parameterized isotropic Gaussian, \ie $a^0\sim\cN(\mu_{\theta}(s),\sigma_{\theta}(s))$. The parameterization trick is then used to express $a^{0}$ as a function of $\theta$. Intuitively, the actor would learn $\theta$ such that the initial distribution is close to the target distribution. Hence, fewer steps are required to converge, as illustrated in \Figref{fig:actor}. Note that if the number of steps $L=0$, \STAC\ is reduced to SAC. Besides, to deal with the non-smooth nature of deep Q-function landscapes which might lead to particle divergence in the sampling process, we bound the particle updates to be within a few standard deviations ($t$) from the mean of the learned initial distribution, \ie $-t\sigma_{\theta} \leq a_{\theta}^{l} \leq t\sigma_{\theta}$, $\forall l \in [1,L]$. Eventually, the initial distribution $q_{\theta}^{0}$ learns to contour the high-density region of the target distribution and the following updates refine it by converging to the spanned modes. Formally, the parameters $\theta$ are computed by minimizing the expected KL-divergence between the policy $q_{\theta}^L$ induced by the particles from the sampler and the EBM of the Q-values:
\vspace{-2mm}
\begin{align} 
\theta^{*} \!=\! &\argmax\nolimits_{\theta}\!  \E_{s_t \sim \cD,  a_{\theta}^L \sim \pi_{\theta}} \left[ Q_{\phi}(s_t, a_{\theta}^L) \right] + \alpha  \E_{s_t \sim \cD} \left[  \mathcal{H}( \pi_{\theta}(\cdot|s_t) ) \right] \nonumber\\
\text{s.t. } &-t\sigma_{\theta} \leq a_{\theta}^{l} \leq t\sigma_{\theta}, \quad \forall l \in [1,L] .
\label{eq:actor_loss}
\end{align}
Here, $\cD$ is the replay buffer. The derivation is in Appendix~\ref{appendix:actor_loss}. Note that the constraint does not truncate the particles as it is not an invertible transformation which then violates the assumptions of the change of variable formula. Instead, we sample more particles than we need and select the ones that stay within the range. We call \STAC($\phi,\theta$) and \STAC($\phi$) as two versions of  \STAC\ with/without the parameterized initial distribution. The complete \STAC\ algorithm is in Algorithm~\ref{alg:stac} of Appendix~\ref{appendix:summary}.

\subsection{A Closed-Form Expression of the Policy's Entropy}
\label{sec:entropy}
\vspace{-1mm}
A critical challenge in MaxEnt RL is how to efficiently compute the entropy term $\mathcal{H}(\pi(\cdot|s_{t+1}))$
in Eq~\eqref{eq:max_entr_obj}. We show that, if we model the policy as an iterative sampler from the EBM, under certain conditions, we can derive a closed-form estimate of the entropy at convergence. 

\begin{theorem}\label{thm:generic_entropy}
Let $F:\mathbb{R}^{n} \rightarrow \mathbb{R}^{n}$ be an invertible transformation of the form $F(a) = a + \epsilon h(a) $. We denote by $q^L(a^L)$ the distribution obtained from repeatedly applying $F$ to a set of samples $\{a^{0}\}$ from an initial distribution $q^{0}(a^0)$ over $L$ steps, \ie $a^{L} = F \circ F \circ \cdots \circ F(a^{0})$. Under the condition $ \epsilon ||\nabla_{a^{l}_{i} } h(a_i)||_{\infty} \ll 1 $, $\forall l \in [1,L]$,
the distribution of the particles at the $L^{\text{th}}$ step is:
\vspace{-2mm}

\begin{equation}\label{eq:generic_entropy}
\log q^{L}(a^L) \approx \log q^0(a^0) - \epsilon \sum\nolimits_{l = 0}^{L-1}  \, \Tr(\nabla_{a^{l}} h(a^{l})) + \cO(\epsilon^2 dL).
\end{equation}
\vspace{-2mm} 
Here, $d$ is the dimensionality of $a$, \ie $a\in \mathbb{R}^d$ and $\cO(\epsilon^2 dL)$ is the order of approximation error.
\end{theorem}
\proofsketch{
As $F$ is invertible, we apply the change of variable formula (Appendix~\ref{appendix:change_of_variable}) on the transformation $F \circ F \circ \cdots F$ and obtain: $\log q^{L}(a^L) = \log q^0(a^0) - \sum\nolimits_{l = 0}^{L-1} \log \left| \det (I+\epsilon \nabla_{a^{l}} h(a^{l})) \right|$.
Under the assumption $ \epsilon ||\nabla_{a_i} h(a_i)||_{\infty} \ll 1 $, we apply the corollary of Jacobi's formula (Appendix~\ref{appendix:jacobi}) and get Eq.~\eqref{eq:generic_entropy}. The detailed proof is in Appendix \ref{app:proof_theorem_41}.} Note that the condition $ \epsilon ||\nabla_{a_i} h(a_i)||_{\infty} \ll 1 $ can always be satisfied when we choose a sufficiently small step size $\epsilon$, \textit{or} the gradient of $h(a)$ is small, \ie $h(a)$ is Lipschitz continuous with a sufficiently small constant. \\
It follows from the theorem above, that the entropy of a policy modeled as an EBM sampler (Eq~\eqref{eq:actor_loss}) can be expressed analytically as:
\begin{equation}\label{eq:closed_form_entropy}
\mathcal{H}(\pi_{\theta}(\cdot|s))\!=\! -\mathbb{E}_{a_{\theta}^{0}\sim q_{\theta}^{0}}\Big[\log q_{\theta}^{L}(a_{\theta}^{L} |s)\Big]\!\approx\!-\mathbb{E}_{a_{\theta}^{0}\sim q_{\theta}^{0} }\Big[\!\log q_{\theta}^{0}(a^{0}|s)\!-\!\epsilon\!\sum\nolimits_{l = 0}^{L-1} \Tr\!\Big(\!\nabla_{a_{\theta}^{l}}\!h(a_{\theta}^{l},s)\!\Big)\!\Big].
\end{equation}
In the following, we drop the dependency of the action on $\theta$ for simplicity of the notation.
\subsection{Invertible Policies}
\label{sec:invertible_policies}
\vspace{-2mm}
Next, we study the invertibility of three popular EBM samplers: SVGD, SGLD, and HMC as well as the efficiency of computing the trace, \ie $\Tr(\nabla_{a^l} h(a^{l}, s))$ in Eq~\eqref{eq:generic_entropy} for the ones that are invertible.
\begin{proposition}[SVGD invertibility]\label{prop:svgd_invertibility}
 Given the SVGD learning rate $\epsilon$ and RBF kernel $k(\cdot,\cdot)$ with variance $\sigma$, if $\epsilon \ll \sigma $, the update rule of SVGD dynamics defined in Eq~\eqref{eq:svgd_update} is invertible.
\end{proposition}
\proofsketch{We use the explicit function theorem to show that the Jacobian $\nabla_{a} F(a,s)$ of the update rule $F(a,s)$ is diagonally dominated and hence invertible. This yields invertibility of $F(a,s)$. See detailed proof in Appendix~\ref{appendix: svgd_invertibility}.}
\begin{theorem}\label{thm:svgd_entropy} 
The closed-form estimate of $\log q^L(a^L|s)$ for the SVGD based sampler with an RBF kernel $k(\cdot,\cdot)$ is 
\vspace{-4mm}
\begin{equation*}\small
\begin{aligned}
\log q^L(a^L|s)\!\approx\!\log\!q^{0}(a^{0}|s)\!+\!\frac{\epsilon}{m \sigma^2}\!\sum_{l=0}^{L-1}\sum_{j=1, a^{l} \neq a_{j}^{l}}^{m}\!k(a^{l}_{j},a^l)\Big(\!(a^{l}\!-\!a^{l}_{j})^{\top} \nabla_{a^{l}_{j}} Q(s,a^{l}_{j})\!+\!\frac{\alpha}{\sigma^2}\|a^{l}\!-\!a^{l}_{j}\|^{2}\!-\!d \alpha \Big).
\end{aligned}\label{eq:svgd_entr_formula}
\vspace{-4mm}
\end{equation*}
\end{theorem}
Here, $(\cdot)^\top$ denotes the transpose of a matrix/vector.
Note that the entropy does not depend on any matrix computation, but only on vector dot products and first-order vector derivatives. The proof is in Appendix~\ref{appendix: svgd_likelihood_proof}. Intuitively, the derived likelihood is proportional to (1) the concavity of the curvature of the Q-landscape, captured by a weighted average of the neighboring particles' Q-value gradients and (2) pairwise-distances between the neighboring particles ($\sim\!\|a^{l}_{i}\!-\!a^{l}_{j}\|^2\cdot \exp{(\|a^{l}_{i}\!-\!a^{l}_{j}\|^2)}$), \ie the larger the distance the higher is the entropy. We elaborate on the connection between this formula and non-parametric entropy estimators in Appendix~\ref{appendix:related_work}.
\begin{proposition}[SGLD, HMC]\label{prop:sgld_invertibility}
The SGLD and HMC updates are not invertible w.r.t. $a$. 
\end{proposition} 
\vspace{-1mm}
\proofsketch{SGLD is stochastic (noise term) and thus not injective. HMC is only invertible if conditioned on the velocity $v$. Detailed proofs are in Appendices~\ref{appendix: sgld_invertibility_proof}-\ref{appendix: hmc_invertibility_proof}.}\\
From the above theoretic analysis, we can see that SGLD update is not invertible and hence is not suitable as a sampler for \STAC. While the HMC update is invertible, its derived closed-form entropy involves calculating Hessian and hence computationally more expensive. Due to these considerations, we choose to use SVGD with an RBF kernel as the underlying sampler of \STAC.

\vspace{-1mm}\section{Results}\vspace{-1mm}
\label{sec:results}
We first evaluate the correctness of our proposed closed-form entropy formula. Then we present the results of different RL algorithms on multigoal and MuJoCo environments.




\begin{figure*}[t!]
\vspace{-2mm}
\centering
    \begin{subfigure}[t]{0.34\linewidth}
        \centering
        \includegraphics[width=1\linewidth]{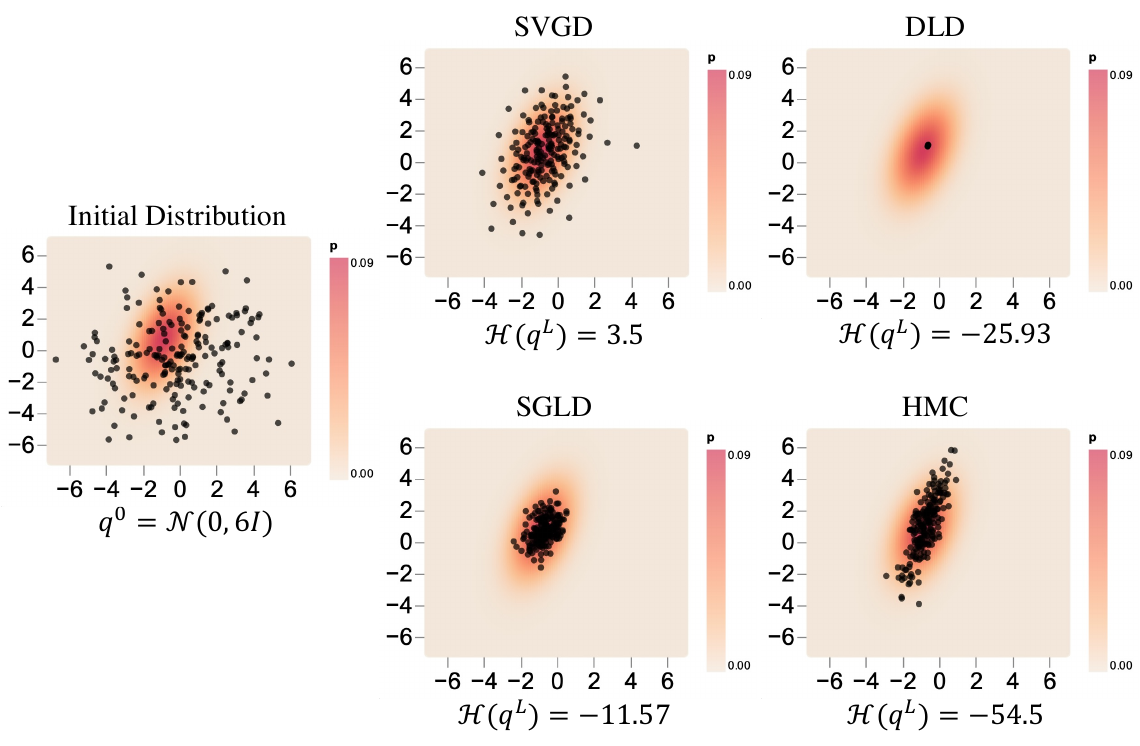}
        \vspace{-6mm}
        \caption{Recovering the GT entropy}
        \label{fig:test_entr}
    \end{subfigure} \hfill
        \begin{subfigure}[t]{0.37\linewidth}
        \centering         
        \includegraphics[width=1\linewidth]{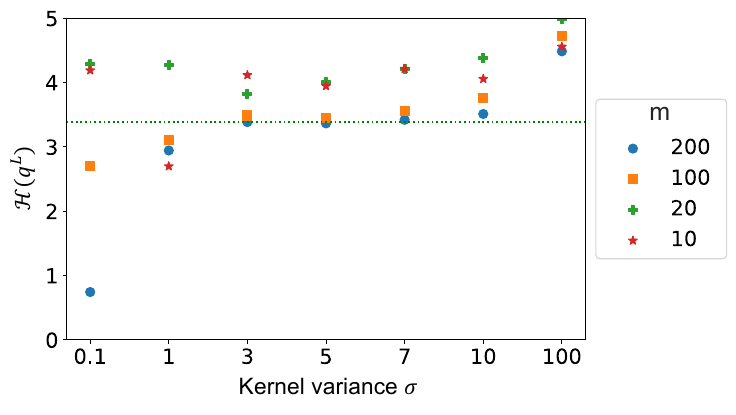}
        \vspace{-6mm}
        \caption{Effect of $\sigma$ on $\cH(q^L)$ 
        }\label{fig:kernel_variance}
    \end{subfigure}\hfill
    \begin{subfigure}[t]{0.28\linewidth}
        \centering
        \includegraphics[width=\linewidth]{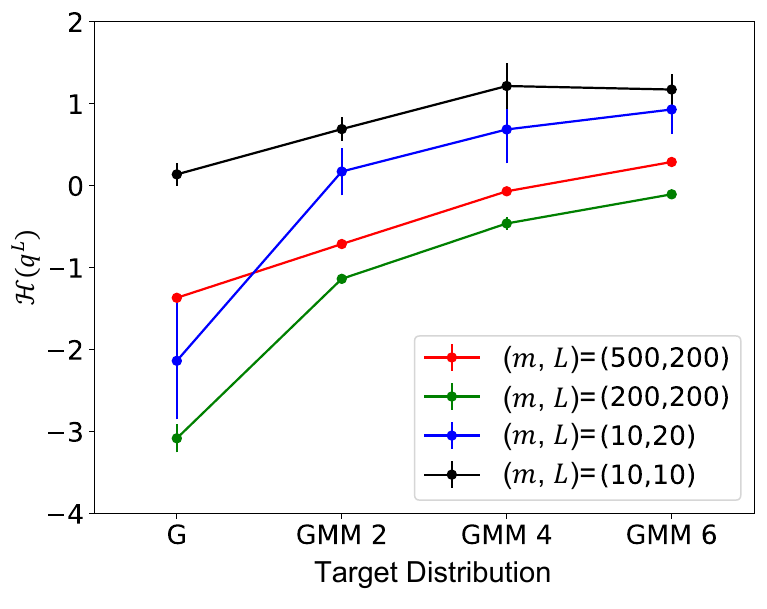}
        \vspace{-6mm}
        \caption{Effect of $m$ and $L$ on $\cH(q^L)$}
        \label{fig:steps_particles}
    \end{subfigure}
    \vspace{-3mm}
\caption{Entropy evaluation results.} 
\label{fig:table3}
\vspace{-8mm}
\end{figure*}

\vspace{-1mm}\subsection{Entropy Evaluation}
\vspace{-2mm}
This experiment tests the correctness of our entropy formula. We compare the estimated entropy for distributions (with known ground truth entropy or log-likelihoods) using different samplers and study the sensitivity of the formula to different samplers' parameters.
\noindent\textbf{(1) Recovering the ground truth entropy.} In \Figref{fig:test_entr}, we plot samples (black dots) obtained by SVGD, SGLD, DLD and HMC at convergence to a Gaussian with ground truth entropy $\cH(p)=3.41$, starting from the same initial distribution (leftmost sub-figure). We also report the entropy values computed via Eq.\eqref{eq:closed_form_entropy}. Unlike SGLD, DLD, and HMC, SVGD recovers the ground truth entropy. This empirically supports Proposition \ref{prop:sgld_invertibility} that SGLD, DLD, and HMC are not invertible. \noindent\textbf{(2) Effect of the kernel variance.} Figure~\ref{fig:kernel_variance} shows the effect of different SVGD kernel variances $\sigma$, where we use the same initial Gaussian from Figure~\ref{fig:test_entr}. We also visualize the particle distributions after $L$ SVGD steps for the different configurations in Figure \ref{fig:kernel_variance_vis} of Appendix~\ref{sec:entr_eval}. We can see that when the kernel variance is too small (\eg $\sigma\!=\!0.1$), the invertibility is violated, and thus the estimated entropy is wrong even at convergence. On the other extreme when the kernel variance is too large (\eg $\sigma\!=\!100$), \ie when the particles are too scattered initially, the particles do not converge to the target Gaussian due to noisy gradients in the first term of Eq.\eqref{eq:svgd_update}. The best configurations hence lie somewhere in between (\eg $\sigma\!\in\!\{3, 5, 7\}$). \noindent\textbf{(3) Effect of SVGD steps and particles.} Figure~\ref{fig:steps_particles} and Figure~\ref{fig_spp_iv} (Appendix.~\ref{sec:entr_eval}) show the behavior of our entropy formula under different configurations of the number of SVGD steps and particles, on two settings: (i) GMM $M$ with an increasing number of components $M$, and (ii) distributions with increasing ground truth entropy values, \ie Gaussians with increasing variances $\sigma$. Results show that our entropy consistently grows with an increasing $M$ (Figure~\ref{fig:steps_particles}) and increasing $\sigma$ (Figure~\ref{fig_spp_iv}), even when a small number of SVGD steps and particles is used (\eg $L=10, m=10$).

\vspace{-1mm}\subsection{Multi-goal Experiments} \vspace{-1mm}
\begin{wrapfigure}{r}{0.24\textwidth} 
    \vspace{-9mm}
    \centering
    \includegraphics[width=\linewidth]{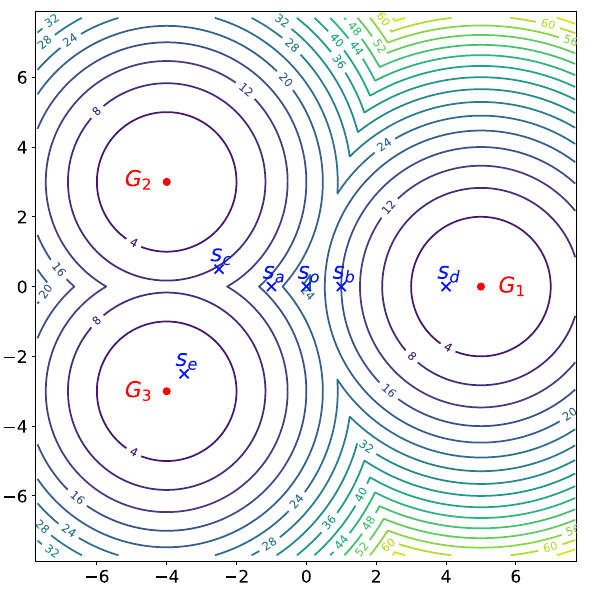} 
    \vspace{-5mm}
    \caption{Multigoal env.}
    \label{fig:multigoal}
    \vspace{-3.5mm}
\end{wrapfigure}
To check if \STAC\ learns a better solution to the max-entropy objective (Eq~\eqref{eq:max_entr_obj}), we design a new multi-goal environment as shown in Figure~\ref{fig:multigoal}. The agent is a 2D point mass at the origin trying to reach one of the goals (in red). Q-landscapes are depicted by level curves. Actions are bounded in $[-1, 1]$ along both axes. Critical states for the analysis are marked with blue crosses. It is built on the multi-goal environment in~\cite{haarnoja2017reinforcement} with modifications such that all the goals have (i) the same maximum expected future reward (positive) but (ii) different maximum expected future entropy. This is achieved by asymmetrically placing the goals (two goals on the left side and one on the right, leading to a higher expected future entropy on the left side) while assigning the same final rewards to all the goals. The problem setup and hyperparameters are detailed in Appendix~\ref{appendix:multi-goal}. \textbf{(1) Multi-modality.} Figure~\ref{fig:multigoal-trajectory} visualizes trajectories (blue lines) collected from 20 episodes of \STAC($\phi,\theta$), \STAC($\phi$), SAC, SQL and SAC-NF (SAC with a normalizing flow policy, \cite{mazoure2020leveraging}) agents (rows) at test time for increasing entropy weights $\alpha$ (columns). \STAC\ and SQL consistently cover all the modes for all $\alpha$ values, while this is only achieved by SAC and SAC-NF for large $\alpha$ values. Note that, in the case of SAC, this comes at the expense of accuracy. Although normalizing flows are 
expressive enough in theory, they are known to quickly collapse to local optima in practice \cite{kobyzev2020normalizing}. The dispersion term in \STAC\ encodes an inductive bias to mitigate this issue. \textbf{(2) Maximizing the expected future entropy.} We also see that with increasing $\alpha$, more \STAC\ and SAC-NF trajectories converge to the left goals ($G_2$/$G_3$). This shows both models learn to maximize the expected future entropy. This is not the case for SQL whose trajectory distribution remains uniform across the goals. SAC results do not show a consistent trend. This validates the hypothesis that the entropy term in SAC only helps exploration but does not lead to maximizing future entropy. The quantified distribution over reached goals is in \Figref{fig:multi-goal-distribution} of Appendix~\ref{appendix:multi-goal}. \textbf{(3) Robustness/adaptability.} To assess the robustness of the learned policies, we place an obstacle (red bar in \Figref{fig:multigoal-robustness}) on the path to $G_2$. We show the test time trajectories of $20$ episodes using \STAC, SAC, SQL and SAC-NF agents trained with different $\alpha$'s. We observe that, for \STAC\ and SAC-NF, with increasing $\alpha$, more trajectories reach the goal after hitting the obstacles. This is not the case for SAC, where many trajectories hit the obstacle without reaching the goal. SQL does not manage to escape the barrier even with higher $\alpha$. Additional results on the \textbf{(4) effect of parameterization of $q^{0}$}, and the \textbf{(5) entropy's effect on the learned Q-landscapes} are respectively reported in \Figref{fig:mean_std_curve_initial_actor} and \Figref{fig:entropy_states} of Appendix~\ref{appendix:multi-goal}.
\vspace{-3.5mm}
\begin{figure*}[h]
\centering
\begin{minipage}{.48\textwidth}
  \centering
  \includegraphics[width=0.85\linewidth]{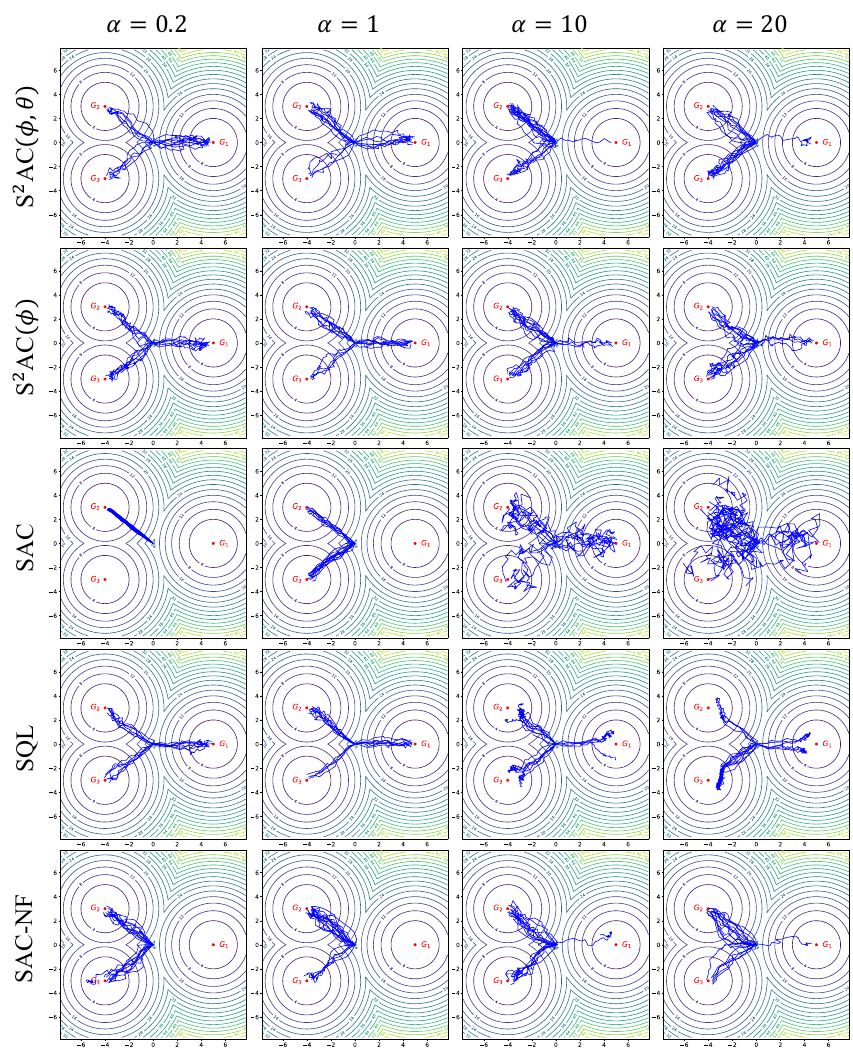}
  \vspace{-2mm}
  \caption{\STAC\ and SAC-NF learn to maximize the expected future entropy (biased towards $G_2/G_3$) while SAC and SQL do not. \STAC\ consistently recovers all modes, while SAC-NF with smaller $\alpha$'s does not, indicating its instability.}
  \vspace{-6mm}
  \label{fig:multigoal-trajectory}
\end{minipage}\hfill
\begin{minipage}{.48\textwidth}
  \centering
  \includegraphics[width=0.85\linewidth]{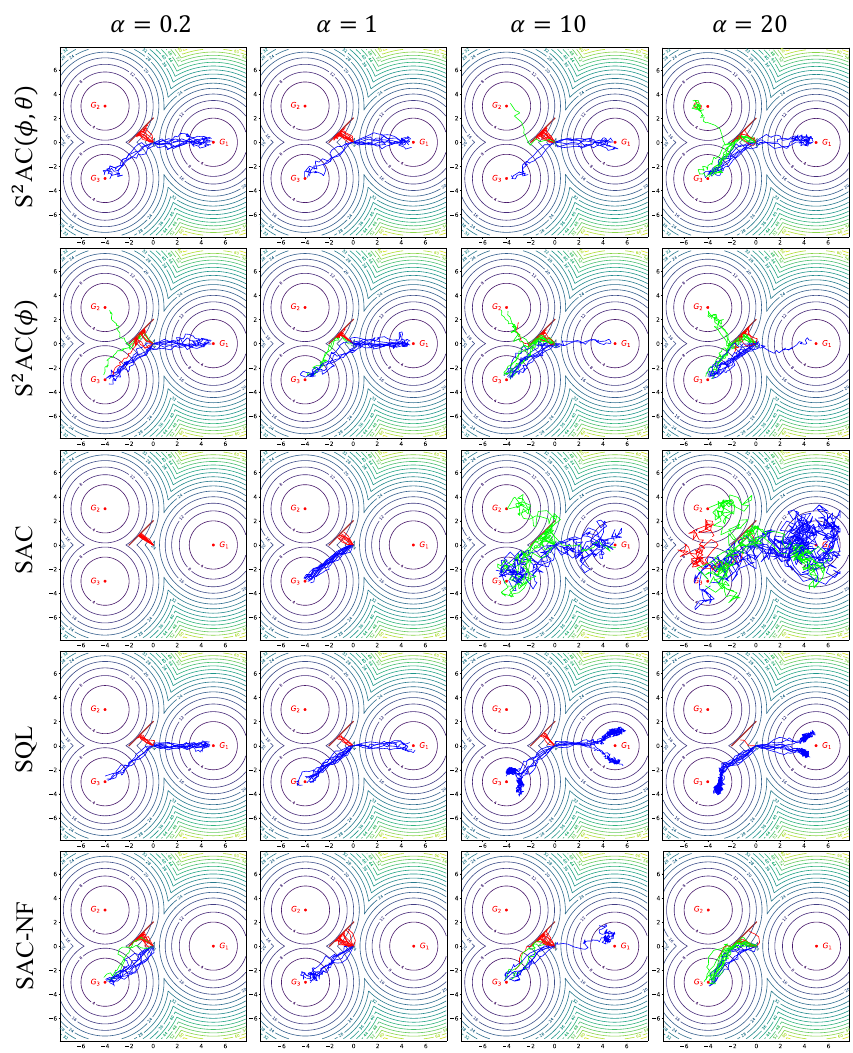}
  \vspace{-2mm}
  \caption{\STAC\ and SAC-NF are more robust to perturbations. Obstacle $O$ is placed diagonally at $[-1,1]$. Trajectories that did and did not reach the goal after hitting $O$ are in green and red, respectively.}
  \vspace{-5mm}
  \label{fig:multigoal-robustness}
\end{minipage}
\end{figure*}

\vspace{2mm}
\subsection{Mujoco Experiments}
\vspace{-1mm}
We evaluate \STAC\ on five environments from MuJoCo \citep{brockman2016openai}: Hopper-v2, Walker2d-v2, HalfCheetah-v2, Ant-v2, and Humanoid-v2. As baselines, we use (1) DDPG \citep{gu2017q}, (2) PPO \citep{schulman2015trust}, (3) SQL \citep{haarnoja2017reinforcement}, (4) SAC-NF \citep{mazoure2020leveraging}, and (5) SAC \citep{haarnoja2018soft}. Hyperparameters are in Appendix \ref{appendix:mujoco}.

\vspace{-1mm}
\textbf{(1) Performance and sample efficiency.} We train five different instances of each algorithm with different random seeds, with each performing $100$ evaluation rollouts every $1000$ environment steps. Performance results are in \Figref{fig:mujoco_3env}(a)-(e). The solid curves correspond to the mean returns over the five trials and the shaded region represents the minimum and maximum. \STAC($\phi,\theta$) is consistently better than SQL and SAC-NF across all the environments and has superior performance than SAC in four out of five environments. Results also show that the initial parameterization was key to ensuring the scalability (\STAC($\phi$) has poor performance compared to \STAC($\phi, \theta$)). \Figref{fig:mujoco_3env}(f)-(j) demonstrate the statistical significance of these gains by leveraging statistics from the rliable library \citep{agarwal2021deep} which we detail in Appendix ~\ref{appendix:mujoco}.
\vspace{-1mm}

\setlength{\tabcolsep}{2.3pt}
\begin{wraptable}{r}{7.0cm} 
\vspace{-3mm}
\begin{footnotesize}
    \begin{tabular}{c|cccc} 
      \toprule
      & Hopper  & Walker2d & HalfCheetah & Ant \\\midrule
      Action dim & 3 & 6 & 6 & 8 \\
      State dim & 11 & 17 & 17 & 111\\
      \midrule
      SAC  & 0.723 & 0.714 & 0.731 & 0.708\\
      SQL  & 0.839 & 0.828 & 0.815 & 0.836\\ 
      \STAC($\phi,\theta$) & 3.267 & 4.622 & 4.583 & 5.917 \\
      \STAC($\phi,\theta, \psi$) & 0.850 & 0.817 & 0.830 & 0.837\\
      \bottomrule
    \end{tabular}
\end{footnotesize}
\vspace{-2mm}
\caption{Action selection run-time on MuJoCo.}
\label{tab:run_time2}
\vspace{-2mm}
\end{wraptable}

\textbf{(2) Run-time.} We report the run-time of action selection of SAC, SQL, and \STAC\ algorithms in Table \ref{tab:run_time2}. \STAC($\phi,\theta$) run-time increases linearly with the action space. To improve the scalability, we train an amortized version that we deploy at test-time, following \citep{haarnoja2017reinforcement}. Specifically, we train a feed-forward deepnet $f_{\psi}(s,z)$ to mimic the SVGD dynamics during testing, where $z$ is a random vector that allows mapping the same state to different particles. Note that we cannot use $f_{\psi}(s,z)$ during training as we need to estimate the entropy in Eq~\eqref{eq:closed_form_entropy}, which depends on the unrolled SVGD dynamics (details in Appendix~\ref{appendix:mujoco}). The amortized version \STAC($\phi,\theta,\psi$) has a similar run-time to SAC and SQL with a slight tradeoff in performance (\Figref{fig:mujoco_3env}).

\begin{figure}
    \centering
    \includegraphics[width=13cm]{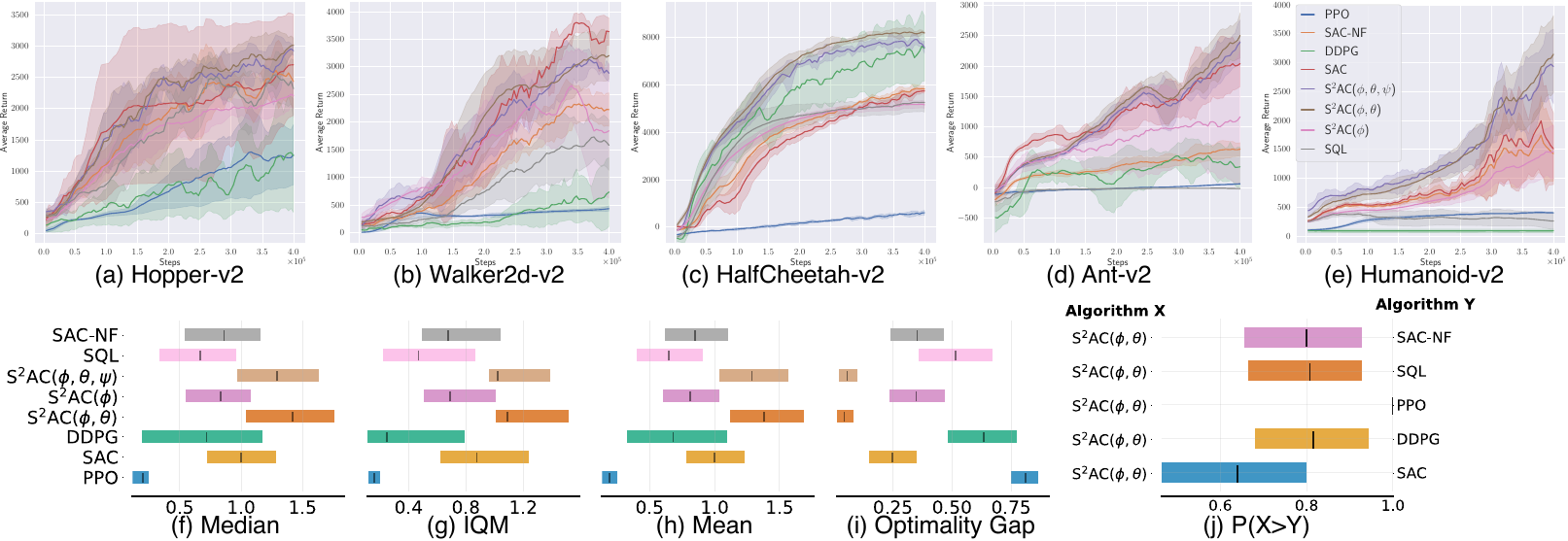}
    \vspace{-2mm}
    \caption{\textbf{(a)-(e)}: Performance curves on the MuJoCo benchmark (training). \STAC\ outperforms SQL and SAC-NF on all environments and SAC on 4 out of 5 environments. \textbf{(f)-(i)}: Comparison of Median, IQM, Mean, and Optimality Gap between \STAC\ and baseline algorithms. \textbf{(j)}: The probabilities of \STAC\ outperforming baseline algorithms.}
    \vspace{-5mm}
    \label{fig:mujoco_3env}
\end{figure}

\section{Related Work}
\label{sec:related_work}
\textbf{MaxEnt RL}~\citep{todorov2006linearly,ziebart2010modeling,rawlik2012stochastic} aims to learn a policy that gets high rewards while acting as randomly as possible. To achieve this, it maximizes the sum of expected future reward and expected future entropy. It is different from entropy regularization~\citep{schulman2015trust,o2016combining,schulman2017proximal} which maximizes entropy at the current time step. It is also different from multi-modal RL approaches \citep{tang2018boosting} which recover different modes with equal frequencies without considering their future entropy. MaxEnt RL has been broadly incorporated in various RL domains, including inverse RL~\citep{ziebart2008maximum,finn2016guided}, stochastic control~\citep{rawlik2012stochastic,toussaint2009robot}, guided policy search~\citep{levine2013guided}, and off-policy learning~\citep{haarnoja2018soft,haarnoja2018soft1}. MaxEnt RL is shown to maximize a lower bound of the robust RL objective~\citep{eysenbach2022maximum} and is hence less sensitive to perturbations in state and reward functions. From the variational inference lens, MaxEnt RL aims to find the policy distribution that minimizes the \textit{KL}-divergence to an EBM over Q-function. The desired family of variational distributions is (1) expressive enough to capture the intricacies of the Q-value landscape (\eg multimodality) and (2) has a tractable entropy estimate. These two requirements are hard to satisfy. SAC~\citep{haarnoja2018soft} uses a Gaussian policy. Despite having a tractable entropy, it fails to capture arbitrary Q-value landscapes. SAC-GMM~\citep{haarnoja2018acquiring} extends SAC by modeling the policy as a Gaussian Mixture Model, but it requires an impractical grid search over the number of components. Other extensions include IAPO \citep{marino2021iterative} which also models the policy as a uni-modal Gaussian but learns a collection of parameter estimates (mean, variance) through different initializations. While this yields multi-modality, it does not optimize a MaxEnt objective. SSPG \citep{cetin2022policy} and SAC-NF \citep{mazoure2020leveraging} respectively improve the policy expressivity by modeling the policy as a Markov chain with Gaussian transition probabilities and as a normalizing flow. Due to training instability, the reported multi-goal experiments in \citep{cetin2022policy} show that, though both models capture multimodality, they fail to maximize the expected future entropy in positive reward setups. SQL \citep{haarnoja2017reinforcement}, on the other hand, bypasses the explicit entropy computation altogether via a soft version of value iteration. It then trains an amortized SVGD \citep{wang2016learning} sampler from the EBM over the learned Q-values. However, estimating soft value functions requires approximating integrals via importance sampling which is known to have high variance and poor scalability. We propose a new family of variational distributions induced by a parameterized SVGD sampler from the EBM over Q-values. Our policy is expressive and captures multi-modal distributions while being characterized by a tractable entropy estimate. \\
\textbf{EBMs} \citep{lecun2006tutorial,wu2018sparse} are represented as Gibbs densities $p(x) = \exp{E(x)}/Z$, where $E(x) \in \mathbb{R}$ is an energy function describing inter-variable dependencies and $Z= \int \exp{E(x)}$ is the partition function. Despite their expressiveness, EBMs are not tractable as the partition function requires integrating over an exponential number of configurations. Markov Chain Monte Carlo (MCMC) methods \citep{van2018simple} (\eg HMC~\citep{hoffman2014no}, SGLD~\citep{welling2011bayesian}) are frequently used to approximate the partition function via sampling. There have been recent efforts to parameterize these samplers via deepnets \citep{levy2017generalizing,gong2018meta,feng2017learning} to improve scalability. Similarly to these methods, we propose a parameterized variant of SVGD \citep{liu2016stein} as an EBM sampler to enable scalability to high-dimensional action spaces. Beyond sampling, we derive a closed-form expression of the sampling distribution as an estimate of the EBM. This yields a tractable estimate of the entropy. This is opposed to previous methods for estimating EBM entropy which mostly rely on heuristic approximation, lower bounds \cite{dai2017calibrating,dai2019kernel}, or neural estimators of mutual information \citep{kumar2018engan}. The idea of approximating the entropy of EBMs via MCMC sampling by leveraging the change of variable formula was first proposed in \cite{dai2019exponential}. The authors apply the formula to HMC and LD, which, as we show previously, violate the invertibility assumption. To go around this, they augment the EBM family with the noise or velocity variable for LD and HMC respectively. But the derived log-likelihood of the sampling distribution turns out to be --counter-intuitively-- independent of the sampler's dynamics and equal to the initial distribution, which is then parameterized using a flow model (details in Appendix \ref{sec:entropy_ebms}). We show that SVGD is invertible, and hence we sample from the original EBM, so that our derived entropy is more intuitive as it depends on the SVGD dynamics.\\
\textbf{SVGD-augmented RL} \citep{liu2016stein} has been explored under other RL contexts. \citet{liu2017steinn} use SVGD to learn a distribution over policy parameters. While this leads to learning diverse policies, it is fundamentally different from our approach as we are interested in learning a single multi-modal policy with a closed-form entropy formula. \citet{castanet2023stein, chen2021variational} use SVGD to sample from multimodal distributions over goals/tasks. We go beyond sampling and use SVGD to derive a closed-form entropy formula of an expressive variational distribution.



\vspace{-2mm}
\section{Conclusion}
\vspace{-2mm}
We propose \STAC, an actor-critic algorithm that yields a more optimal solution to the MaxEnt RL objective than previously proposed approaches. \STAC\ achieves this by leveraging a new family of variational distributions characterized by SVGD dynamics. The proposed distribution has high expressivity, \ie it is flexible enough to capture multimodal policies in high dimensional spaces, and a tractable entropy estimate. Empirical results show that \STAC\ learns expressive and robust policies while having superior performance than other MaxEnt RL algorithms. For future work, we plan to study the application of the proposed variational distribution to other domains and develop benchmarks to evaluate the robustness of RL agents.

\section*{Acknowledgments}
Bo An is supported by the National Research Foundation Singapore and DSO National Laboratories under the AI Singapore Programme (AISGAward No: AISG2-GC-2023-009). Haipeng Chen is supported by William \& Mary FRC Faculty Research Grants.

\bibliography{rl}
\bibliographystyle{iclr2024_conference}

\newpage

\appendix
\onecolumn

\begin{center}
\textbf{\Large Supplementary Material}
\end{center}

\section{Summary}\label{appendix:summary}
In this paper, we propose a new variational distribution that we use to model the actor in the context of actor-critic MaxEntr RL algorithms. Our distribution is induced by an SVGD sampler with a parametrized initial distribution (isotropic Gaussian). It enables fitting multi-modal distribution (\eg EBM) and is characterized by a closed-form entropy estimate. Hence, it addresses the major bottleneck in classical MaxEntr RL algorithms. Our derivation is based on the unique invertibility property of the SVGD sampler, which is not satisfied for other popular samplers (\eg SGLD, HMC). The key to achieving scalability was to learn the initial Gaussian distribution such that it contours the high-density region of the target distribution, by limiting particles' updates to be always within few standard deviations of the mean of this Gaussian. This resulted in better and faster exploration of the relevant regions of the target distribution. Our proposed approach \STAC\ is summarized in Algorithm~\ref{alg:stac}.The rest of the supplementary is organized as follows:
\begin{itemize}
\item Appendix~\ref{appendix:related_work} provides additional related work on the entropy estimation.
\item Appendix~\ref{appendix:preliminary} introduces additional preliminaries on EBM samplers, the change of variable formula and the Jacobi formula. 
\item Appendix~\ref{appendix:MaxEnt_sol} provides the derivation of the optimal policy for the MaxEntr RL objective.
\item Appendix~\ref{appendix:actor_loss} provides the derivation of the actor objective. 
\item Appendices~\ref{app:proof_theorem_41}-\ref{appendix:closed_form_entropy} provide proofs for theorems related to (1) a generic closed-form expression of log-likelihood of inverible samplers, (2) discussion of samplers invertibility and (3) closed-form likelihood derivation for SVGD. 
\item  Appendices~\ref{sec:entr_eval}-\ref{appendix:mujoco} provide additional results for the (1) entropy evaluation, (2) multigoal environment, and (3) MuJoCo environments.
\end{itemize}

\begin{algorithm}[H]
\caption{Stein Soft Actor Critic (\STAC)}
\begin{algorithmic}[1]
\STATE Initialize parameters $\phi$, $\theta$, hyperparameter $\alpha$, and replay buffer $\mathcal{D} \leftarrow \emptyset$ 
\FOR{each iteration}
\FOR{each environment step $t$}
\STATE Sample action particles $\{a\}$ from $\pi_{\theta}(\cdot|s_t)$
\STATE Select $a_t \in \{a\}$ using  exploration strategy
\STATE Sample next state $s_{t+1} \sim p(s_{t+1}|s_{t},a_{t})$
\STATE Update replay buffer $\mathcal{D} \leftarrow \mathcal{D} \cup (s_t,a_t,r_t,s_{t+1})$
\ENDFOR
\FOR{each gradient step}
\STATE \textbf{Critic update:}
\STATE \hspace{4mm} Sample particles $\{a\}$ from an EMB sampler $\pi_{\theta}(\cdot|s_{t+1})$
\STATE \hspace{4mm} Compute entropy $\mathcal{H}(\pi_{\theta}(\cdot|s_{t+1}))$ using Eq.\eqref{eq:closed_form_entropy} 
\STATE \hspace{4mm} Update $\phi$ using Eq.\eqref{eq:critic_update}
\STATE \textbf{Actor update:}
\STATE \hspace{4mm} Update $\theta$ using Eq.\eqref{eq:actor_loss}
\ENDFOR
\ENDFOR
\end{algorithmic}
\label{alg:stac}
\end{algorithm}


\newpage
\section{Additional Related-work}
\label{appendix:related_work}

\subsection{Entropy}\label{appendix:additional_related_work}
The differential entropy \cite{cover1999elements,shannon2001mathematical} of a $p$-dimensional random variable $X$ with a probability density function $p(x)$ is defined by: $H(p) = - \int p(x) \ln{p(x)} dx$. The differential entropy plays a central role in information and communication theory, statistics \cite{1448714}, signal processing \cite{comon, learned2003ica}, machine learning and pattern recognition \cite{mannor2005cross, Rubinstein, hino2010conditional, liu2022combating, wulfmeier2015maximum}. For example, Max-Entropy RL \cite{wulfmeier2015maximum, haarnoja2017reinforcement, haarnoja2018soft} methods augment the expected reward objective with an entropy maximization term which results in learning multi-modal policies and more robustness. Recently, Liu \etal \cite{liu2022combating} propose maximizing the entropy of the discriminator distribution to combat mode collapse. In statistical mechanics entropy appears as the negative of the rate function to quantify the fluctuations around thermodynamic equilibrium \cite{roldan2021quantifying}. Estimating the differential entropy for expressive distributions is a challenging problem as it requires computing a closed-form expression of the probability density function. Several non-parametric approaches \cite{beirlant1997nonparametric, GYORFI1987425, paninski2003estimation, NIPS2008_ccb09896} based on approximating the entropy using samples $\cD=\{x_i\}_{i=1}^{|\cD|}$ from $p(x)$, have been proposed in the literature. These methods can be classified into (1) plug-in estimates \cite{1055550, ivanov1981properties, Joe} which approximate $p(x)$ via a kernel density estimate, (2) samples spacing \cite{BEIRLANT1985300,Noel,dudewicz1981entropy, HALL1986201} and (3) nearest-neighbor distances based estimates \cite{bernhofen1996ranking,Peter,kozachenko1987sample} which express the entropy in terms of pairwise distances between the samples (larger distances imply higher entropy). Next, we review the work on entropy estimation of Energy-Based-Models (EBMs).

\subsection{Entropy of EBMs}\label{sec:entropy_ebms}
In this work, we are interested in computing entropy estimates for the class of EBMs \cite{lecun2006tutorial} represented as Gibbs densities $p(x) = \frac{\exp{E(x)}}{Z}$, where $E(x) \in \mathbb{R}$ is an energy function describing inter-variable dependencies and $Z= \int \exp{E(x)}$ is an intractable partition function. EBMs provide a unified framework for many probabilistic and non-probabilistic approaches, particularly for learning and inference in structured models and are widely used in computer science (\eg semantic segmentation, colorization, image generation, inverse optimal control, collaborative filtering) \cite{salakhutdinov2007restricted,messaoud2018structural,zhao2021learning,gao2020learning,xie2020generative,Zheng_2021_CVPR,carleo2017solving,messaoud2020can,xie2021learning,pang2021trajectory,messaoud2021toward,xie2016theory,xie2018cooperative,xie2021learning1,xie2022tale} and physics \cite{carleo2017solving,gao2017efficient,torlai2018neural,melko2019restricted} (\eg to model the wavefunctions of quantum systems). To estimate the entropy of EBMs, previous methods mostly rely on heuristic approximation, lower bounds \cite{dai2017calibrating,dai2019kernel}, or neural estimators of mutual information to approximate the entropy \cite{kumar2018engan}. The idea of approximating the entropy of EBMs via the one from an MCMC sampler by leveraging the change of variable formula was first proposed by \cite{dai2019exponential}. Specifically, the authors apply the formula to HMC and LD which, as we show in Appendix.~\ref{appendix:invertibility_proof}, violate the invertibility assumption. To go around this, the authors propose augmenting the EBM family with the noise or velocity variable for, respectively, LD and HMC, \ie sampling from $p(x)$ is replaced with sampling from $p(x,v)$ or $p(x,\xi)$. The authors assume that the sampler update rule is invertible with respect to the augmented samples $(x,v)$ and $(x,\xi)$. However, computing the determinant of the update rule with respect to the augmented variable is always equal to 1 in this case. Hence, the resulting log-likelihood of the sampling distribution is, counter-intuitively, independent of the sampler's dynamics and equal to the initial distribution, \ie $\log q^{L}(a^L) = \log q^{0}(a^0)$, which the author model using a flow model.  
Differently, we show that SVGD is invertible, our entropy depends on the dynamics of SVGD, we still sample from the original EBM $p(x)$ and our initial distribution is a simple Gaussian. Similarly to the non-parametrized entropy estimates described above, our formula leverages pairwise distances between the neighboring samples. Differently, our formula is also based on the curvature of the energy function $E(x)$ (measured by a weighted average of neighboring particle gradients $\nabla_{x}E(x)$). Hence maximizing our derived entropy results in the intuitive effect of learning smoother energy landscapes.

\newpage
\section{Additional Preliminaries}\label{appendix:preliminary}
In the following, we review (1) additional samplers for EBMs, (2) the change of variable formula and (3) the corollary of the Jacobi's formula.
\subsection{Additional Samplers for EBMs}\label{appendix:additional_samplers}
\textbf{SGLD}~\citep{welling2011bayesian} is a popular Markov chain Monte Carlo (MCMC) method for sampling from a distribution. It initializes a sample $a^0$ from a random distribution, and then in each step $l+1$ it adds the gradient of the current proposal distribution $p(a)$ to the previous sample $a^l$, together with a Brownian motion $\xi\!\sim\! N(0,I)$. 
We denote the step size as $\epsilon$. The iterative update for SGLD is:
\begin{equation}\label{eq:sgld_update}
a^{ l+ 1} = a^{l} +  \epsilon\nabla_{a^{l}} \, \log p( a^{l} ) + \sqrt{2\epsilon}\xi.
\end{equation}

\textbf{DLD} are equivalent to SGLD without the noise term, \ie
\begin{equation}\label{eq:dld_update}
a^{ l+ 1} = a^{l} +  \epsilon\nabla_{a^{l}} \, \log p( a^{l} ).
\end{equation}

\textbf{HMC} is another popular variant of MCMC samplers. The most commonly used discretized Hamilton’s equations are the leapfrog method~\citep{neal2011mcmc}. The three (half) steps of leapfrog updates in HMC are:
\begin{equation}\label{eq:hmc_update}
\begin{aligned}
v^{l+1/2} &= v^l + (\epsilon/2) \nabla_{a} \log p(a^{l+1}) \\
a^{l+1} &= a^l + \epsilon v^{l+1/2} \\
v^{l+1} &= a^{l+1} + (\epsilon/2) \nabla_{a} \log p(a^{l+1})
\end{aligned}
\end{equation}
Here $v^{l}$ is interpreted the velocity at iteration $l$ (assuming unit mass) and $a^l$ is the ``location'' of a sample in a distribution. 

\subsection{Change of Variable Formula}\label{appendix:change_of_variable}
We first introduce the concept of an invertibile function.
\begin{definition}[Invertibile transformation]\label{def:invertible}
Transformation $F: Z\rightarrow X$ is invertible iff $F(\cdot)$ is bijective, \ie simultaneously injective and surjective:
(i) $F(\cdot)$ is injective iff for any $z,z'\in Z$, $F(z)=F(z')\Rightarrow z=z'$;
(ii) $F(\cdot)$ is surjective iff  for every $x\in X$, there exists some $z\in Z$ such that $F(z)=x$.
\end{definition}

According to change of variable formula, the following holds when $F: Z \rightarrow X $ is an invertible function:
\begin{equation*}
\small
p_X(x)=p_Z(z)\Big| {\det\!\frac {\partial F^{-1}(x)} {\partial x}\Big| } = p_Z(z) \Big| \det\!\frac {\partial F(z)} {\partial z}  \Big|^{-1}
\end{equation*}

\subsection{A Corollary of Jacobi's Formula}\label{appendix:jacobi}
An important corollary of Jacobi's Formula \citep{magnus2019matrix} states that, given an invertible matrix $A$, the following equality holds:
\begin{equation*}\small
\log(\det{A}) =\Tr\left(\log A \right) = \Tr \Big( \!\sum\nolimits_{k = 1}^{\infty}\!\!{(- 1)}^{k + 1}\frac{{(A\  - I)}^{k}}{k}\! \Big). 
\end{equation*}
The second equation is obtained by taking the power series of $\log A$.
Hence, under the assumption $\|A-I\|_{\infty}\ll 1$, we obtain:
\begin{equation*}
\log (\det A) \approx \text{tr}(A-I).
\end{equation*}

\newpage
\section{Derivation of the MaxEnt RL optimal policy}
\label{appendix:MaxEnt_sol}
In this section, we prove that the solution $\pi^{*}$ of the MaxEnt RL objective 
\begin{equation}
\max_{\pi}  J(\pi) \equiv \sum_{t} \mathbb{E}_{(s_t,a_t)\sim \rho_{\pi} } \Big[  \gamma^t  \Big( r(a_t,s_t) - \alpha \log \pi(\cdot|s_t)  \Big)  \Big]
\end{equation}
is $\pi^{*} = \frac{  \exp( \frac{1}{\alpha} Q(s,a))}{Z}$. Here, $Q(s,a)$ is the soft Q-function defined as
\begin{eqnarray}
Q(s,a)& = &\mathbb{E}_{(s_t,a_t)\sim \rho_{\pi} } \Big[ \sum_{t} \gamma^t  \Big( r(a_t,s_t) - \alpha \log p(\pi(\cdot|s_t))  \Big)  | s_0=s, a_0=a  \Big] \nonumber \\
& = &  r(a,s)+ \alpha \cH(\pi(\cdot|s)) + \mathbb{E}_{\pi(a'|s)\rho_{\pi}(s')} \Big[Q(s',a')\Big].
\end{eqnarray}

Consequently, we deduce that $\pi^{*}$ is also the solution of the expected KL divergence:
\begin{equation}\label{eq:ebm_}
\pi^{*} = \argmin\nolimits_{\pi} \sum_t \mathbb{E}_{s_t \sim \rho_\pi}\Big[ D_{K\!L}\big(\pi(\cdot|s_t) \| \exp\!{(  Q(s_t,\cdot)/\alpha )}/Z \big)\Big].
\end{equation}

\begin{proof}
We express the MaxEnt loss as a function of $Q(s,a)$, \ie $J(\pi) = \mathbb{E}_{(s,a)\sim \rho_{\pi} } \Big[ Q(s,a) \Big]$. To find $\pi^{*} = \arg\max_{\pi} J(\pi)$ under the constraint $\int_a \pi(a|s)da=1$, we evaluate the Lagrangian (with $\lambda\in\mathbb{R}$ being the Lagrange multiplier):
\begin{equation}
\cL(\pi,\lambda)=  \mathbb{E}_{(s,a)\sim \rho_{\pi} } \Big[ Q(s,a) \Big] + \lambda \Big(\int_{a} \pi(a|s) da - 1 \Big),
\end{equation}
and compute $\frac{\partial \cL(\pi,\lambda)  }{\partial \pi(a|s)}$:
\begin{eqnarray*}
\frac{\partial \cL(\pi,\lambda)  }{\partial \pi(a|s)}\!&\!=\!&\!\frac{\partial}{\partial \pi(a|s) } \Big( \int_{s} \int_{a} \pi(a|s) \rho_{\pi}(s) Q(s,a) da \, ds + \lambda \Big(\int_{a} \pi(a|s) da - 1 \Big)  \Big)\\
\!&\!=\!&\!\frac{\partial}{\partial \pi(a|s) } \Big(\pi(a|s) \rho_{\pi}(s) \Big( r(a,s)- \alpha \log \pi(a|s) + \mathbb{E}_{\pi(a'|s)\rho_{\pi}(s')} \Big[Q(s',a')\Big]\Big) \Big)+ \lambda  \\
\!&\!=\!&\!  \rho_{\pi}(s) \Big( r(a,s)+ \mathbb{E}_{\pi(a'|s)\rho_{\pi}(s')} [Q(s',a')]\Big) - \alpha \rho_{\pi}(s) \frac{\partial  } {\partial \pi(a|s) } \Big( \pi(a|s) \log \pi(a|s) \Big)+ \lambda \\
\!&\!=\!&\!  \rho_{\pi}(s) \Big( r(a,s)+ \mathbb{E}_{\pi(a'|s)\rho_{\pi}(s')} [Q(s',a')]\Big) - \alpha \rho_{\pi}(s) \Big( \log \pi(a|s)  + 1 \Big)+ \lambda. \\
\end{eqnarray*}
Setting $\frac{\partial \cL(\pi,\lambda)  }{\partial \pi(a|s)}$ to 0: 
\begin{eqnarray}
\frac{\partial \cL(\pi,\lambda)  }{\partial \pi(a|s)} =0 &\Longleftrightarrow & \Big( r(a,s)+ \mathbb{E}_{\pi(a'|s)\rho_{\pi}(s')} [Q^{\pi}(s',a')]\Big) - \alpha + \frac{\lambda}{\rho_{\pi}(s)} =\alpha \log \pi(a|s) \nonumber \\
& \Longleftrightarrow &\frac{1}{\alpha} \Big( r(a,s)+ \mathbb{E}_{\pi(a'|s)\rho_{\pi}(s')} [Q(s',a')]\Big) - 1 + \frac{\lambda}{\alpha\rho_{\pi}(s)} = \log \pi(a|s) \nonumber\\
& \Longleftrightarrow & \pi(a|s) = \frac{ \exp \Big(\frac{1}{\alpha} \Big( r(a,s)+ \mathbb{E}_{\pi(a'|s)\rho_{\pi}(s')} [Q(s',a')]\Big) \Big) }{   \exp \Big( 1- \frac{\lambda}{\alpha\rho_{\pi}(s)} \Big) } \nonumber\\
& \Longleftrightarrow & \pi(a|s) = \frac{ \exp \Big(\frac{1}{\alpha} \Big( r(a,s) + \cH(\pi(\cdot|s))+ \mathbb{E}_{\pi(a'|s)\rho_{\pi}(s')} [Q(s',a')]\Big) \Big) }{   \exp \Big( 1- \frac{\lambda}{\alpha\rho_{\pi}(s)} \Big) } \nonumber\\
& \Longleftrightarrow & \pi(a|s) = \frac{ \exp \Big(\frac{1}{\alpha} Q(s,a) \Big) }{   \exp \Big(  \frac{\cH(\pi(\cdot|s))}{\alpha} + 1- \frac{\lambda}{\alpha\rho_{\pi}(s)} \Big) } 
\end{eqnarray}

We choose $\lambda$ such that $\int_{a} \pi(a|s) da=1 $, \ie

\begin{equation}
\int_{a}\!\frac{\exp \Big(\frac{1}{\alpha} Q(s,a) \Big) }{\exp \Big( \frac{\cH(\pi(\cdot|s))}{\alpha}\!+\!1\!-\!\frac{\lambda}{\alpha\rho_{\pi}(s)} \Big) } da\!=\!1\!\Longleftrightarrow\!\lambda\!=\!-\alpha \rho_{\pi}(s) \Big(\log\!\int_{a}\exp\Big(\frac{1}{\alpha} Q(s,a)\Big)da\!-\!\frac{\cH(\pi(\cdot|s))}{\alpha}\!-\!1\Big).
\end{equation}
Hence, $\pi^{*}(a|s) = \exp( \frac{1}{\alpha} Q(s,a))/ \int_{a}\exp(\frac{1}{\alpha} Q(s,a))$. A similar proof follows for any state and action pairs. Trivially, $\pi^{*}$ is also the global minimum of Eq.\eqref{eq:ebm_}.
\end{proof}

\section{Derivation of the actor objective (Eq.\eqref{eq:actor_loss})}
\label{appendix:actor_loss}
In the following, we prove that the objective

\begin{equation*}
     \arg\min_{\theta} \mathbb{E}_{s_t \sim \cD} \Big[ D_{KL} \Big(\pi_{\theta}(\cdot|s_t) \Big|\Big| \exp \Big( \frac{1}{\alpha} Q_{\phi}(s_t,\cdot)\Big)/Z(\phi) \Big) \Big] \\
\end{equation*}
is equivalent to
\begin{equation*}
     \arg\max_{\theta} \mathbb{E}_{s_t \sim \cD, a_t\sim\pi_{\theta}(a_t|s_t)} \Big[ Q_{\phi}(s_t, a_t)\Big] + \mathbb{E}_{s_t}\Big[\alpha \cH(\pi_{\theta}(a_t|s_t))  \Big], 
\end{equation*}
with $\cD$ being a replay buffer.

\begin{proof}
    \begin{eqnarray*}
        \theta^{*} \!&\!=\!&\! \arg\min_{\theta} \mathbb{E}_{s_t \sim \cD} \Big[ D_{KL} \Big(\pi_{\theta}(\cdot|s_t) \Big|\Big| \exp \Big( \frac{1}{\alpha} Q_{\phi}(s_t,\cdot)\Big)/Z(\phi) \Big) \Big] \\
        \!&\!=\!&\! \arg\min_{\theta} \mathbb{E}_{s_t \sim \cD, a_t\sim\pi_{\theta}(a_t|s_t)} \Big[ \log(\pi_{\theta}(a_t|s_t)) - \Big( \frac{1}{\alpha} Q_{\phi}(s_t, a_t) -\log Z(\phi)  \Big)  \Big] \\
        \!&\!=\!&\! \arg\min_{\theta} \mathbb{E}_{s_t \sim \cD, a_t\sim\pi_{\theta}(a_t|s_t)} \Big[ \log(\pi_{\theta}(a_t|s_t)) -  \frac{1}{\alpha} Q_{\phi}(s_t, a_t)   \Big] \\
        \!&\!=\!&\! \arg\max_{\theta} \mathbb{E}_{s_t \sim \cD, a_t\sim\pi_{\theta}(a_t|s_t)} \Big[ - \log(\pi_{\theta}(a_t|s_t)) +  \frac{1}{\alpha} Q_{\phi}(s_t, a_t)   \Big] \\
        \!&\!=\!&\! \arg\max_{\theta} \mathbb{E}_{s_t \sim \cD, a_t\sim\pi_{\theta}(a_t|s_t)} \Big[  \frac{1}{\alpha} Q_{\phi}(s_t, a_t) +\cH(\pi_{\theta}(a_t|s_t))  \Big] \\
        \!&\!=\!&\! \arg\max_{\theta} \mathbb{E}_{s_t \sim \cD, a_t\sim\pi_{\theta}(a_t|s_t)} \Big[ Q_{\phi}(s_t, a_t)\Big] + \mathbb{E}_{s_t \sim \cD}\Big[\alpha \cH(\pi_{\theta}(a_t|s_t))  \Big] \\
    \end{eqnarray*}
\end{proof}

\newpage
\section{Proof of Theorem~\ref{thm:generic_entropy}}
\label{app:proof_theorem_41}

\begin{theorem*}
Let $F:\mathbb{R}^{n} \rightarrow \mathbb{R}^{n}$ be an invertible transformation of the form $F(a) = a + \epsilon h(a) $. We denote by $q^L(a^L)$ the distribution obtained from repeatedly ($L$ times) applying $F$ to a set of action samples (called ``particles'') $\{a^{0}\}$ from an initial distribution $q^{0}(a^0)$, \ie $a^{L} = F \circ F \circ \cdots \circ F(a^{0})$. Under the condition $ \epsilon ||\nabla_{a_i} h(a_i)||_{\infty} \ll 1 $, the closed-form expression of $\log q^{L}(a^L)$ is:
\begin{equation}
\log q^{L}(a^L) = \log q^0(a^0) - \epsilon \sum_{l = 0}^{L-1}  \, \Tr(\nabla_{a^{l}} h(a^{l})).
\end{equation}
\end{theorem*}

\begin{proof}
Based on the change of variable formula (Appendix~\ref{appendix:change_of_variable}), when for every iteration $l \in[1,L]$, the transformation $a^{l}=L(a^{l-1})$ (of the action sampler in our paper) is invertible, we have:
\begin{align*}
q^{l}( a^{l} ) = q^{l-1}( a^{l-1} )\left|\det\frac{da^{l}}{da^{l-1}} \right|^{-1} ,\forall l \in[1,L].
\end{align*}

By induction, we derive the probability distribution of sample $a^{L}$:   
\begin{align*}
q^{L}( a^{L} ) = q^0( a^0 )\prod_{l = 1}^{L}\left| \det\frac{da^{l}}{da^{l-1}} \right|^{-1} = q^0( a^0 )\prod_{l = 0}^{L-1}\left| \det \big(I+\epsilon\nabla_{a^{l}}h(a^{l})\big) \right|^{-1}
\end{align*}

By taking the $\log$ for both sides, we obtain:
\begin{align*}
\log q^{L}(a^L) = \log q^0(a^0) - \sum_{l = 0}^{L-1} \log \left| \det\big( I+ \epsilon\nabla_{a^{l }}h(a^{l }) \big) \right|.
\end{align*}

Let $A=I+ \epsilon\nabla_{a^{l }}h(a^{l })$, under the assumption $ \epsilon ||\nabla_{a_i} h(a_i)||_{\infty} \ll 1 $, \ie $||A-I||_{\infty} \ll 1$,  we apply the corollary of Jacobi's formula (Appendix~\ref{appendix:jacobi}) and get
\begin{align*}
\log q^{L}(a^L) &\approx \log q^0(a^0) - \sum_{l = 0}^{L-1} \, \Tr\big( (I+ \epsilon\nabla_{a^{l }}h(a^{l }))-I) \big) + \cO(\epsilon^2 d L)\\
&\approx \log q^0(a^0) - \epsilon \sum_{l = 0}^{L-1}  \, \Tr\left(\nabla_{a^{l}} h(a^{l})\right) + \cO(\epsilon^2 d L).
\end{align*}
Here, $d$ is the action space dimension.
\end{proof}

\newpage
\section{Samplers Invertibility Proofs}
\label{appendix:invertibility_proof}

We start by state the implicit function theorem which we will be using in the following proofs.
\begin{theorem}[\textbf{Implicit function theorem}]
Let $f: \mathbb{R}^{n} \rightarrow \mathbb{R}^{n}$ be continuously
differentiable on some open set containing $a$, and suppose $\det{(Jf(a))} = \det{(\nabla_{a}f(a))} \neq 0$. Then, there is some open set $V$ containing $a$ and an open $W$ containing $f(a)$ such
that $f: V \rightarrow W$ has a continuous inverse $f^{-1}: W \rightarrow V$ which is differentiable $\forall y \in W$.

\label{thm:inverse_function} 
\end{theorem}

\subsection{Stochastic Gradient Langevin Dynamics}
\label{appendix: sgld_invertibility_proof}

\begin{proposition*}[SGLD]
The SGLD update in Eq.\eqref{eq:sgld_update} is not invertible. 
\end{proposition*}

\begin{proof}
We show that SGLD are not invertible using two different methods: (1) We show that SGLD is not a bijective transformation, (2) Using the implicit function theorem, we show that the Jacobian of the dynamics is not invertible.

\subsubsection{Method1: SGLD is not invertible $\iff$ SGLD is not a bijection} 

The update rule $F(\cdot)$ for SGLD and DGLD are given by Eq.\eqref{eq:sgld_update} and Eq.\eqref{eq:dld_update}, respectively. In the following, we drop the dependency on the time step for ease of notation.

\textbf{Injectivity} is equivalent to checking: $F( a_{1} ) = F( a_{2} ) \Longrightarrow a_{1} = a_{2}$.
This, however, doesn't hold in case of SGLD as the noise terms $\xi_1$ and $\xi_2$ can be chosen such that the equality 
$$a_{1} + \ \epsilon\nabla_{a_{1}}\log p\left( a_{1} \right) + \sqrt{2\epsilon}\xi_{1}\ ={\ a}_{2} + \ \epsilon\nabla_{a_{2}}\log p\left( a_{2} \right) + \sqrt{2\epsilon}\xi_{2}$$ holds with $a_1 \neq a_2$. 
Therefore, SGLD is not injective. The same holds for DGLD, where the equality $a_{1} + \nabla_{a_{1}}\log p\left( a_{1} \right) =
a_{2} + \nabla_{a_{2}}\log p\left( a_{2} \right) $ can be valid for $a_1 \neq a_2$. A counter-example: $a_{1}=a_{2} + \eta$ and
$\eta + \nabla_{a_{1}}\log p\left( a_{1} \right) = \nabla_{a_{2}}\log p\left( a_{2} \right)$, with $\eta$ being an arbitrary constant. 

\textbf{Surjectivity} is equivalent to checking:  \( \forall a^{l+1}\in \mathbb{R}^{d}, \exists a^{l}\in \mathbb{R}^{d}\) s.t.
\(a^{l+1} = F( a^{l})\). Assume that
\(a^{l+1} = \ a^{l} + \ \epsilon\nabla_{a^{l}}\log p\left( a^{l} \right)\),
and \(a^{l}\in \mathbb{R}^{d},\) we can always choose an
adaptive learning rate \(\epsilon\) such that
\(\epsilon\nabla_{a^{l}}\log p\left( a^{l} \right)\)
\(= \ {a^{l+1}  -  a}^{l}\ \).

\subsubsection{Method2: Implicit function theorem } 
We compute the derivative of the update rule in Eq~\ref{eq:sgld_update} with respect to $a$: $J_F = I + \epsilon \nabla^{2}_{a} \log p(a) $. It's possible for $J_F$ not to be invertible, \eg in case $I = -\epsilon \nabla^{2}_{a} \log p(a)$. Hence, in general $F(a)$ is not guaranteed to be a bijection.
\end{proof}

\subsection{Hamiltonian Monte Carlo (HMC)}
\label{appendix: hmc_invertibility_proof}
\begin{proposition*}[HMC] The HMC update in Eq.\eqref{eq:hmc_update} is not invertible w.r.t. $a$. 
\end{proposition*}

\cite{neal2011mcmc} show that HMC update rule is only invertible with respect to the $(a,v)$, \ie when conditioning on $v$. Since $v$ is sampled from a random distribution, it has the effect of the noise variable in SGLD. Hence, a similar proof applies.



\subsection{Stein Variational Gradient Descent}
\label{appendix: svgd_invertibility}

\begin{proposition*}[SVGD]
Under the assumption that $\epsilon \ll \sigma $, the update rule of SVGD dynamics defined in Eq.\eqref{eq:svgd_update} with an RBF kernel is invertible.
\end{proposition*}

\subsubsection{Method2: SVGD is invertible $\iff$ SVGD is a bijection}

\textbf{Injectivity.} The equality $F(a_1)=F(a_2)$:
\[ a_{1} + \frac{\epsilon}{m}\sum_{j = 1}^{m}{\lbrack k\left( a_{j},a_{1} \right)\nabla_{a_{j}}\log p\left( a_{j} \right) - \nabla_{a_{j}}k\left( a_{j},a_{1} \right)\rbrack\ \text{~}}\!=\!{\ a}_{2} + \frac{\epsilon}{m}\sum_{l = 1}^{m}{\lbrack k\left( a_{l},a_{2} \right)\nabla_{a_{l}}g\left( a_{l} \right) - \nabla_{a_{l}}k\left( a_{l},a_{2} \right)\rbrack\ \text{~}}\]
is too complex to hold for a solution other than $a_1 = a_2$ given the sum over multiple particles on both sides and the dependency on the kernel.

$\Longrightarrow \text{not~obvious~(~depends~on~the~Kernel)}$

\textbf{Surjectivity.} Similarly, to Langevin dynamics, surjectivity can be achieved by choosing a suitable learning rate.

\subsubsection{Method 2: Implicit function theorem}\label{appendix:invertability_1d_proof}
We start by proving the proposition above for the 1-Dimensional case, \ie $a\in \mathbb{R}$. Then, we extend the proof to the multi-dimensional case, \ie $a\in \mathbb{R}^d$.\\
\textbf{1-Dimensional Case.} We prove that $F$ is invertible by showing that $F$ is bijective, which is equivalent to showing that $F$ is strictly monotonic, \ie $\nabla_{a_{i}} F(a_i) > 0$ or $\nabla_{a_{i}} F(a_i) < 0, \quad \forall a_i$. \\
Computing the derivative of the SVGD update (Eq.~\ref{eq:svgd_update}) rule w.r.t $a_i$ results in: 
$$
\nabla_{a_{i}} F(a_i)=1+\frac{\varepsilon}{m} \sum_{i=1}^m \nabla_{a_{i}} k\left(a_i, a_j\right) \nabla_{a_{j}} g\left(a_j\right)+\nabla_{a_{i}} \nabla_{a_j} k\left(a_i, a_j\right).
$$
For $k(a_i, a_j)=e^{-\frac{\|a_i-a_j\|^2}{2\sigma^2}}$, we have: 
$\begin{cases}
    &\nabla_{a_{j}} k(a_i, a_j)=\frac{-2(a_i-a_j)}{2 \sigma^2} k(a_i, a_j)=\frac{-(a_i-a_j)}{\sigma^2} k(a_i, a_j)\\
    &\nabla_{a_{j}} k(a_i, a_j)=\frac{(a_i-a_j)}{\sigma^2} k(a_i, a_j)\\
    &\nabla_{a_{i}} \nabla_{a_{j}}=\frac{1}{\sigma^2} k(a_i,a_j)\left(1-\frac{1}{\sigma^2}\|a_i-a_j\|^2\right)
\end{cases}$

Hence, 
\begin{equation*}
\nabla_{a_i} F(a_i) = 1 + \frac{\epsilon}{m} \sum_{i=1}^{m} \frac{k(a_i,a_j)}{\sigma^2} \left( -(a_{i} -a_{j})  \nabla_{a_j}\log p(a_j) +    1 - \frac{ \|a_i - a_j\|  }{\sigma^2}  \right). 
\end{equation*}

Next, under the condition $\epsilon < \sigma$, we show that $\nabla_{a_i} F(a_i)>0$, $\forall a_i$.
This is equivalent to showing that $\nabla_{a_i} F(a_i)>-1$.

\begin{eqnarray*}
\nabla_{a_i} F(a_i)>-1 
\iff  \frac{\epsilon}{m \sigma^2} \sum_{j=1}^{m} k(a_i,a_j) \left(  -(a_i-a_j) \nabla_{a_j} \log p_{a_j}(a_j) +1- \frac{\| a_i -a_j \|^2}{\sigma^2}   \right)>-1
\end{eqnarray*}

We compute a lower bound on the LHS and investigate when it's strictly larger than $-1$. We can safely assume that $-3 \sigma \leq k(a_i,a_j) (a_i  - a_j) \leq 3 \sigma$ and $-3 \sigma \leq k(a_i,a_j) \|a_i  - a_j\|^2\leq 3 \sigma$. We compute the lower bound as:
$  \sum_{j=1}^m \frac{\epsilon  \alpha}{m \sigma^2} \left(  -3 \sigma \|\nabla_{x_j} \log p(x_j)\| + 1 - \frac{(3 \sigma)^2}{\sigma^2}   \right) =  \frac{\epsilon  \alpha}{m \sigma^2}  \left( -3 \sigma  \sum_{j=1}^m \|\nabla_{x_j} \log p(x_j)\| - 8 m \right) $. This results in: 

\begin{equation}
\nabla_{a_i} F(a_i)\!>\!\frac{\epsilon  \alpha}{m \sigma^2}  \left( -3 \sigma  \sum_{j=1}^m \|\nabla_{a_j} \log p(a_j)\|\!-\!8 m \right)\!>\!-1 \iff \sum_{j=1}^{m} \|\nabla_{a_j} \log p(a_j)\|\!<\!\frac{ m \sigma }{3 \epsilon \alpha }\!-\!\frac{8m}{3 \sigma}
\end{equation}

Hence, $ \max_{a_j} \|\nabla_{a_j} \log p(a_j)\|\!<\!\frac{  \sigma }{3 \epsilon \alpha }\!-\!\frac{8}{3 \sigma}$. The LHS is guaranteed to be a large positive number when $\epsilon \ll\sigma$.

\textbf{Multi-Dimensional Case.} We assume that $\log p(a_j)$ is continuously differentiable. Note that in practice, this can be easily satisfied by choosing the activation function to be Elu instead of Relu. We can easily show that:
\begin{eqnarray*}
\nabla_{a_i} F(a_i) = I + \frac{\epsilon}{m \sigma^2} \sum_{i=1}^{m} k(a_i,a_j)
\left( - \nabla_{a_j}\log p(a_j) (a_i - a_j)^{\top} - \frac{1}{\sigma^2} (a_i - a_j) (a_i - a_j)^{\top} + I \right)
\end{eqnarray*}

Next, we will show that $\nabla_{a_i} F(a_i)$ is diagonally dominated and is, hence, invertible, \ie $\det(\nabla_{a_i} h(a_i)) \neq 0$. For this, we show that $\nabla_{a_i} h(a_i)|_{kl} <1 $, $\forall k,l \in [1,d]$. 
\begin{equation*}
\nabla_{a_i} h(a_i)|_{kl} = \frac{1}{m} \sum_{i=1}^{m} k(a_i,a_j)  \left( -\partial_{a_j^{(k)}} \log p(a_j) (a_i^{(l)} - a_j^{(l)}) - (a_i^{(k)} - a_j^{(k)}) (a_i^{(l)} - a_j^{(l)})+1 \right)
\end{equation*}

Following the proof in Section~\ref{appendix:invertability_1d_proof} for the 1-Dimensional case, we show that $ \nabla_{a_i} h(a_i)|_{kl} \ll 1 $ if $\sigma \ll \epsilon$.

\newpage 
\section{Derivation of Closed-Form Likelihood for Samplers}\label{appendix:closed_form_entropy}
\subsection{Proof of Theorem~\ref{thm:svgd_entropy}}
\label{appendix: svgd_likelihood_proof}
\begin{theorem*} 
The closed-form estimate of the log-likelihood $\log q^L(a^L|s)$ for the SVGD-based sampler with an RBF kernel $k(\cdot,\cdot)$ is 
\begin{small}
\begin{equation*}
\log q^L(a^L|s)\approx\log q^{0}(a^{0}|s)-\frac{\epsilon}{m \sigma^2}\! \sum_{l=0}^{L-1}\sum_{ \substack{j=1\\ a^l \neq a^{l}_{j} }  }^{m}
k(a^{l}_{j}, a^l)
\left( -(a^{l}-a^{l}_{j})^{\top} \nabla_{a^{l}_{j}} Q(s,a^{l}_{j}) -\frac{\alpha}{\sigma^2}\|a^{l}-a^{l}_{j}\|^{2} +d \alpha \right),
\end{equation*}
\end{small}
where $d$ is the feature space dimension.
\end{theorem*}

\begin{proof}
We generate a chain of samples using SVGD starting from $a^{0}\sim q^{0}$, and following the update rule $a^{l+1}_{i} \leftarrow a^{l}_{i} + \epsilon \, h(a^{l}_{i}, s)$, where $h(a_{i}^{l},s) = \mathbb{E}_{a^{l}_{j} \sim q^{l} }\Big[ k(a^{l}_{i}, a^{l}_{j}) \nabla_{a_{j}^{l}} Q(s,a_{j}^{l}) + \nabla_{a_{j}^{l}} k(a_{i}^{l}, a_{j}^{l}) \Big]$ and $k(a_{i}^{l}, a_{j}^{l}) = \exp{(- \frac{\|a^{l}_{i}-a^{l}_{j}\|^{2}}{2\sigma^2})}$. This update rule is the optimal direction in the reproducing kernel Hilbert space of $k(\cdot,\cdot)$ for minimizing the KL divergence objective (actor loss):
\begin{equation}
\pi^{*} = \argmin\nolimits_{\pi} \sum_t \mathbb{E}_{s_t \sim \rho_\pi}\Big[ D_{K\!L}\big(\pi(\cdot|s_t) \| \exp\!{(  Q(s_t,\cdot)/\alpha )}/Z \big)\Big].
\end{equation}




According to Proposition~\ref{prop:svgd_invertibility}, the iteration step (Eq.\eqref{eq:svgd_update}) is invertible. Hence, following Theorem~\ref{thm:generic_entropy} and substituting $h(a_{i}^{l},s)$ with the above formula for SVGD, for each action particle $a^{L}_{i}$ we obtain:
\begin{equation*}
\log q^{L}(a^L_i) \approx \log q^0(a^{0}_{i})-\frac{1}{m}\sum_{l=0}^{L-1} \sum_{ \substack{j=1 \\ a_{i}^{l} \neq a_{j}^{l} } }^{m} \Big[ \underbrace{\Tr\Big( \nabla_{a^{l}_{i}} (k(a^{l}_{i}, a^{l}_{j}) \nabla_{a^{l}_{j}} Q(s,a^{l}_{j}))  \Big)}_{\textcircled{1}}+ \underbrace{\Tr \alpha \Big( \nabla_{a^{l}_{i}} \nabla_{a^{l}_{j}} k(a^{l}_{i},a^{l}_{j})  \Big)}_{\textcircled{2}} \Big].
\end{equation*}
Note that we empirically approximate the expectation in $h(a_{i}^{l},s)$ by an empirical mean over particles that are different from $a_{i}^{l}$, in order to avoid computing Hessians in the derivation below. Next we compute simplifications for terms \textcircled{1} and \textcircled{2} respectively. In the following, we denote by $(\cdot)^{(k)}$ the $k$-th dimension of the vector. \\\\
\textbf{Term \textcircled{1}:}
\begin{eqnarray*}
\Tr\left( \nabla_{a^{l}_{i}} (k(a^{l}_{j}, a^{l}_{j}) \nabla_{a^{l}_{j}} Q(s,a^{l}_{j}))  \right)& = & \Tr\left(  \nabla_{a^{l}_{i}} k(a^{l}_{j}, a^{l}_{j}) (\nabla_{a^{l}_{j}} Q(s,a^{l}_{j}))^{\top}   +  k(a^{l}_{j}, a^{l}_{j}) \nabla_{a^{l}_{i}} \nabla_{a^{l}_{j}}Q(s,a^{l}_{j}))   \right)\\
&=& \sum_{t=1}^{d}  \frac{\partial k(a^{l}_{j}, a^{l}_{j})}{\partial(a^{l}_{i})^{(t)}} \frac{\partial Q(s,a^{l}_{j})}{\partial(a^{l}_{i})^{(t)}} + 0\\ 
&=& (\nabla_{a^{l}_{i}} k(a^{l}_{j}, a^{l}_{j}) )^{\top} \nabla_{a^{l}_{j}} Q(s,a^{l}_{j})\\
&=& -\frac{\alpha}{\sigma^2} k(a^{l}_{j}, a^{l}_{j}) (a^{l}_{i}-a^{l}_{j})^{\top} \nabla_{a^{l}_{j}} Q(s,a^{l}_{j})
\end{eqnarray*}

\textbf{Term \textcircled{2}:}
\begin{eqnarray*}
\Tr\left(\nabla_{a^{l}_{i}} \nabla_{a^{l}_{j}} k(a^{l}_{i},a^{l}_{j})\right)& = & \alpha\Tr\left(\nabla_{a^{l}_{i}} \left(\frac{1}{\sigma^2} \, k(a^{l}_{i},a^{l}_{j}) (a^{l}_{i}-a^{l}_{j}) \right)\right) \\
& = &  \frac{\alpha}{\sigma^2}\Tr \left( \nabla_{a^{l}_{i}} k(a^{l}_{i},a^{l}_{j}) (a^{l}_{i}-a^{l}_{j})^{\top}  + k(a^{l}_{i},a^{l}_{j})\cdot I \right)\\
& = &  \frac{\alpha}{\sigma^2}\Tr \left( -\frac{1}{\sigma^2} k(a^{l}_{i},a^{l}_{j}) (a^{l}_{i}-a^{l}_{j})(a^{l}_{i}-a^{l}_{j})^{\top} + k(a^{l}_{i},a^{l}_{j}) \cdot I \right)\\
& = &  \frac{\alpha}{\sigma^2} \,  \sum_{t=1}^{d} \left( -\frac{1}{\sigma^2} k(a^{l}_{i},a^{l}_{j})  (a^{l}_{i}-a^{l}_{j})^{(t)} (a^{l}_{i}-a^{l}_{j})^{(t)}+ k(a^{l}_{i},a^{l}_{j})\right)\\
& = &  -\frac{\alpha}{\sigma^4}\times k(a^{l}_{i}, a^{l}_{j}) \|a^{l}_{i}-a^{l}_{j}\|^{2} + \frac{\alpha}{\sigma^2} \times d \times k(a^{l}_{i},a^{l}_{j}) \\
& = &  k(a^{l}_{i},a^{l}_{j}) \left(-\frac{\alpha}{\sigma^4}\|a^{l}_{i}-a^{l}_{j}\|^{2} + \frac{d\alpha}{\sigma^2} \right)
\end{eqnarray*}


By combining \textbf{Terms \textcircled{1}} and \textbf{\textcircled{2}}, we obtain:
\begin{equation*}
\log q^{L}(a^L_i)\!\approx\!\log p^{0}(a^{0}_{i})\!-\!\frac{\epsilon}{m\sigma^2} \sum_{l=0}^{L-1}\sum_{j=1}^{m} k(a^{l}_{j}, a^{l}_{j}) \left(-(a^{l}_{i}-a^{l}_{j})^{\top} \nabla_{a^{l}_{j}} Q(s,a^{l}_{j})\!-\!\frac{\alpha}{\sigma^2}\|a^{l}_{i}-a^{l}_{j}\|^{2}\!+\!d\alpha \right)\\
\end{equation*}
Proof done if we take a generic action particle $a_i$ in place of $a$.
\end{proof}

\newpage

\section{Additional Results: Entropy Evaluation}
\label{sec:entr_eval}

The SVGD hyperparameters for this set of experiments are summarized in Table~\ref{tab:tb1}. We include additional figures for (1) the effect of 
(2) the kernel variance (\Figref{fig:kernel_variance_vis}) and (2) number of SVGD steps and particles (\Figref{fig:fig_inc_ent}).

\begin{table}[htb]
\centering
\caption{\label{tab:tb1} Parameters}
\setlength\tabcolsep{7pt}
\begin{tabular}{c|c|c}
\toprule
& Parameter & Value   \\\midrule
Figure~\ref{fig:test_entr}-\ref{fig:kernel_variance}& Target distribution & $p\!=\!\cN([-0.69,0.8],[[1.13,0.82],[0.82,3.39]])$\\
& Initial distribution & $q^{0}=\cN([0,0], 6 \bm{I})$\\ \midrule
Figure~\ref{fig:steps_particles}& Target distribution & $p_{\text{GMM}_{M}}\! =\! \sum_{i=1}^{M} \cN([0,0], 0.1 \bm{I})/M$\\
 & Initial distribution & $q^{0}=\cN([0,0], 6 \bm{I})$\\ \midrule
Default& Learning rate & $\epsilon=0.5$  \\
SVGD& Number of steps& $L=200$ \\
parameters& Number of particles & $m=200$ \\ 
& Kernel variance& $\sigma=5$  \\
\bottomrule
\end{tabular}
\end{table}

\begin{figure}[htb]
\centering
\includegraphics[width=\linewidth]{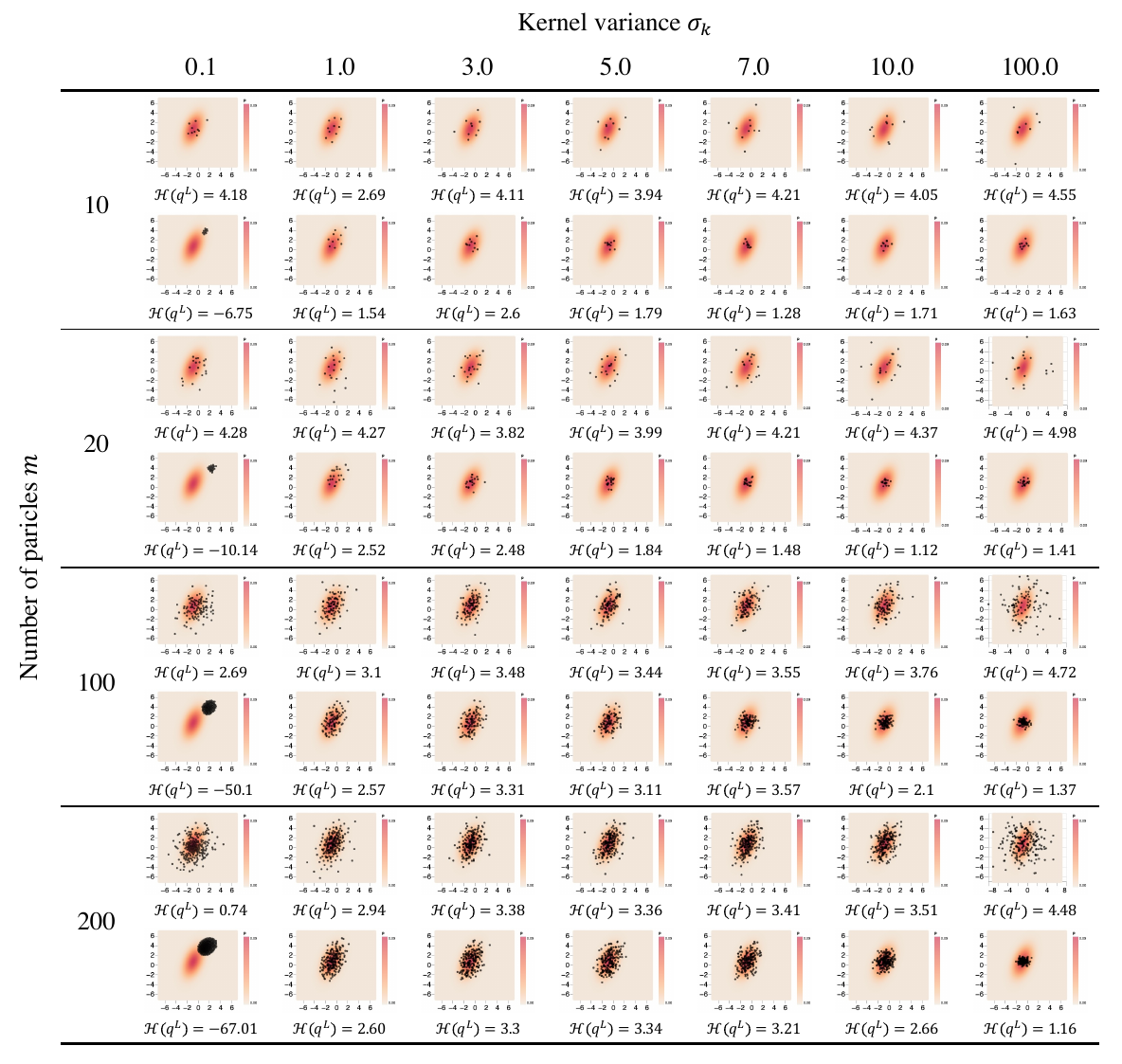}
\caption{ Visualization of the particles after $L$ steps of SVGD for the different configurations of kernel variance $\sigma$ and number of particles $m$ in Figure~\ref{fig:kernel_variance}. }
\label{fig:kernel_variance_vis}
\end{figure}

\begin{figure}[htb]
\centering
    \begin{subfigure}[t]{0.49\linewidth}
        \centering
        \includegraphics[width=0.7\linewidth]{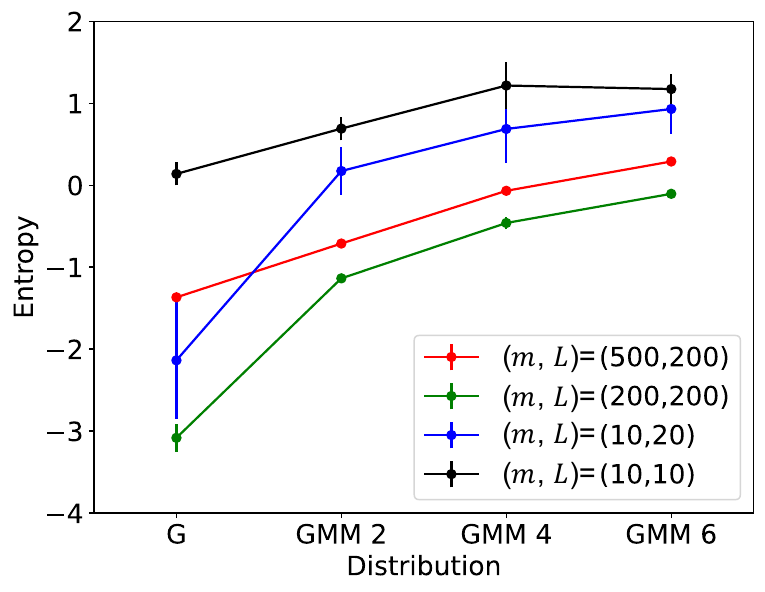}
        \caption{GMM distributions}\label{fig_spp_igmm}
    \end{subfigure} \hfill
    \begin{subfigure}[t]{0.49\linewidth}
        \centering         
        \includegraphics[width=0.7\linewidth]{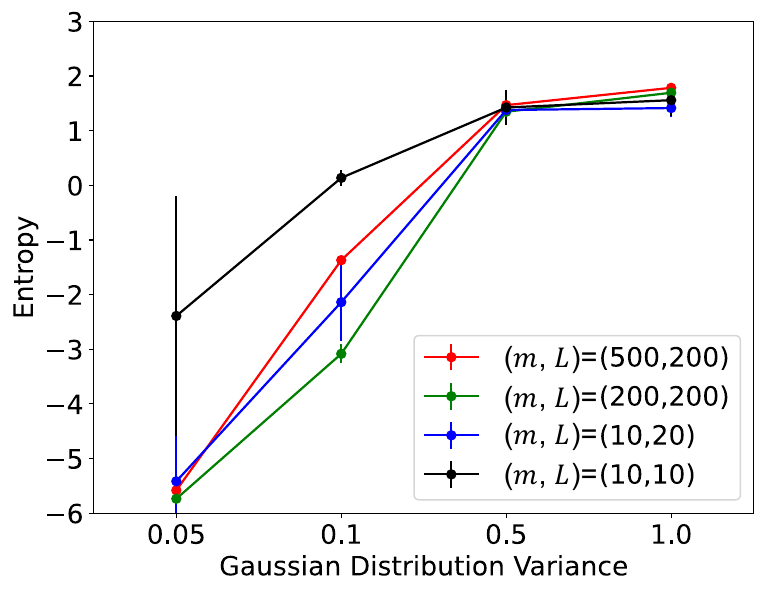}
        \caption{Gaussian distributions}\label{fig_spp_iv}
    \end{subfigure}%
\vspace{-1mm}
\caption{Sensitivity of our entropy formula to the  number of  SVGD steps ($L$) and particles ($m$). Our entropy consistently increases with increasing $\sigma$ and increasing number of GMM components, even when a small number of SVGD steps and particles is used \eg $L=10, m=10$.} 
\label{fig:fig_inc_ent}
\end{figure}

\newpage



\section{Additional Results: Multi-goal Results}
\label{appendix:multi-goal}
Hyperparameters are reported in \tabref{tab:STAC_multigoal_hyperparams}. Additionally, we include results for (1) the effect of the parametrization of the initial distribution (\Figref{fig:entropy_states}), (2) the entropy heatmap (\Figref{fig:entropy_heatmap}), (3) the effect of the entropy on the learned Q-landscapes (\Figref{fig:entropy_states} and \Figref{fig:smoothness}), 
(4) the robustness/adaptability of the learned policies (\Figref{fig:obstacles}) and (5) Amortized \STAC\ results (\Figref{fig:amortized_multigoal}). 

\begin{table}[htb]
\caption{Hyperparameters for multi-goal environment.}
\centering
\setlength\tabcolsep{7pt}
\begin{tabular}{c|c|c}
\toprule
 & Hyperparameter & Value   \\\midrule
 \multirow{3}{*}{Training} & Optimizer & Adam  \\
& Learning rate  & $3 \cdot 10^{-4}$ \\
& Batch size & $100$   \\
\midrule
\multirow{3}{*}{Deepnet} & Number of hidden layers (all networks) & $2$ \\
& Number of hidden units per layer & $256$  \\
& Nonlinearity & ReLU\\
\midrule
\multirow{4}{*}{RL} & Discount factor $\gamma$  & $0.8$ \\
& Replay buffer size $|\cD|$  & $10^{6}$ \\
& Target smoothing coefficient & $0.005$ \\
& Target update interval& $1$ \\
\midrule
\multirow{5}{*}{SVGD} & initial distribution $q^0$& $\cN(\bm{0}, 0.3\!\bm{I})$ \\
& Learning rate $\epsilon$ &  $0.01$ \\
& Number of steps $L$ & $10$ \\
& Number of particles $m$ & $10$\\ 
& Particles range (num. std) $t$ & 3\\ 
& Kernel variance & $\sigma = \frac{ \sum_{i,j} \| a_i - a_j \|^2  }{ 4(2 \log{m+1}) }$\\\bottomrule
\end{tabular}
\label{tab:STAC_multigoal_hyperparams}
\end{table}

\begin{figure}[H]
\centering
    \begin{subfigure}[t]{0.4\linewidth}
        \centering
        \includegraphics[width=\linewidth]{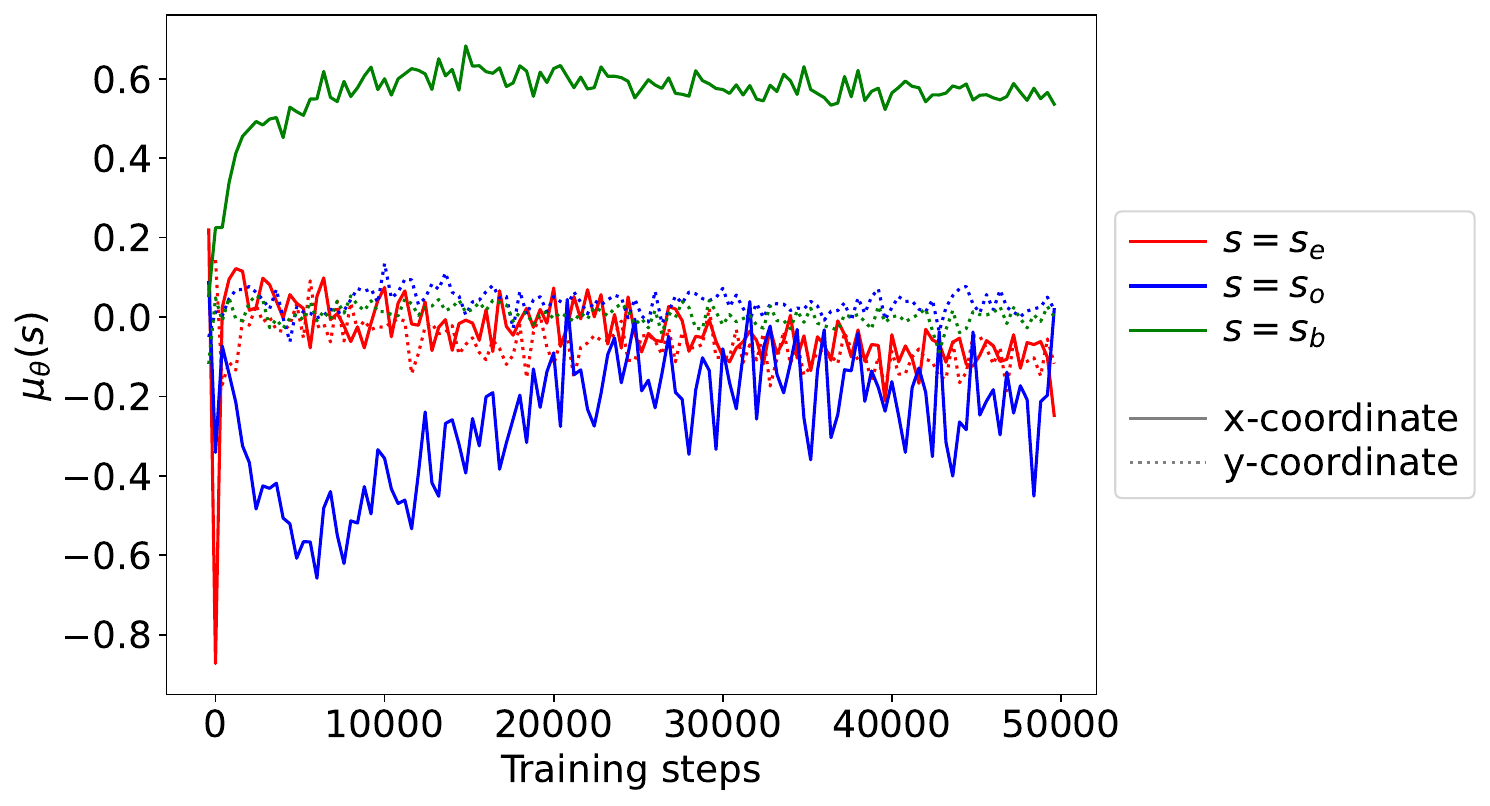}
        \caption{Mean $\mu_{\theta}(s)$}\label{fig:mean_curve_initial_actor}
    \end{subfigure} \hfill
    \begin{subfigure}[t]{0.4\linewidth}
        \centering         
        \includegraphics[width=\linewidth]{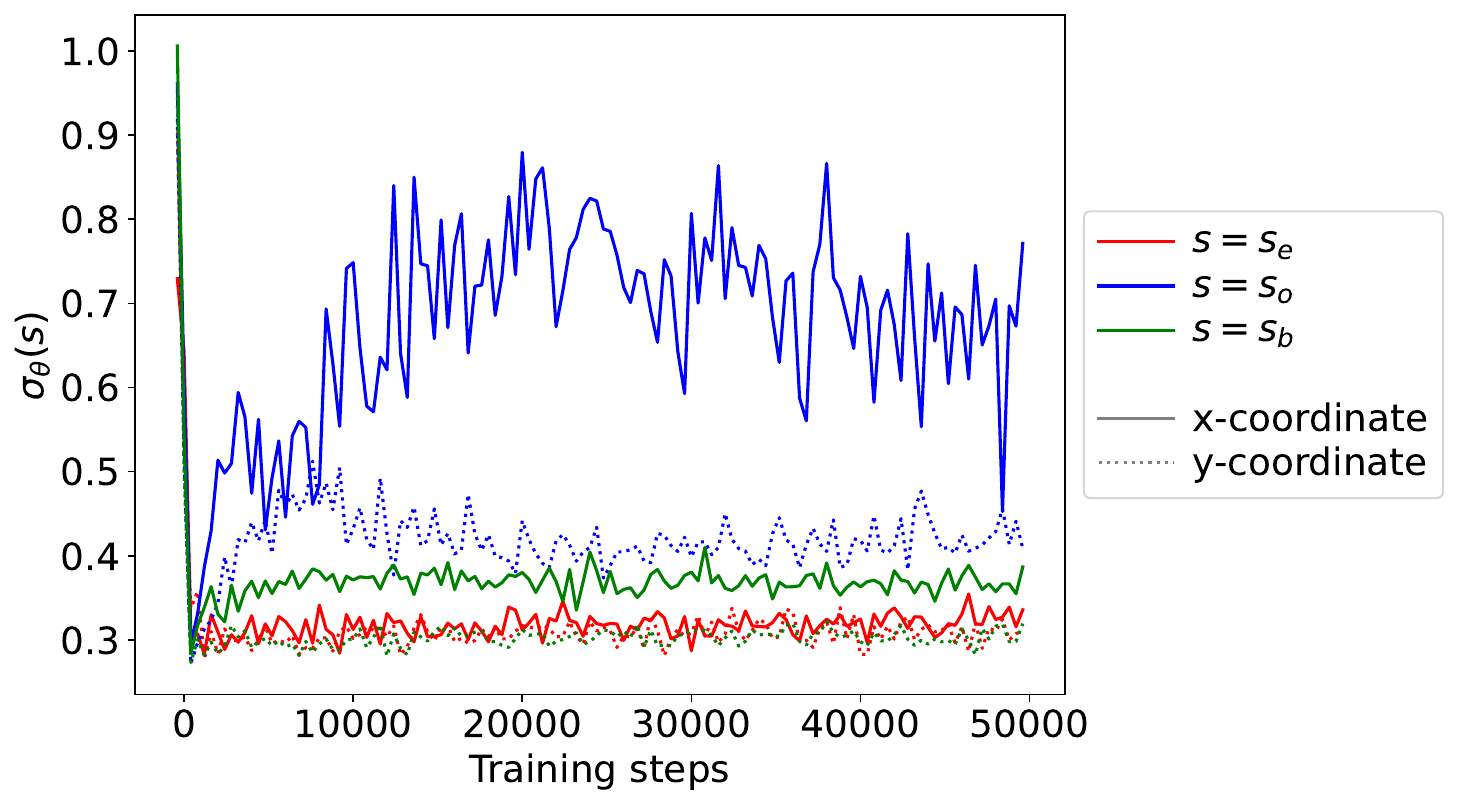}
        \caption{Standard deviation $\sigma_{\theta}(s)$}\label{fig:std_curve_initial_actor}
    \end{subfigure}%
\caption{Trends of $x$ and $y$ coordinates for the mean and standard deviation of the parameterized initial distribution for some critical states, during training.} 
\label{fig:mean_std_curve_initial_actor}
\end{figure}

\textbf{Distribution of reached goals for the multi-goal environment.} 
\Figref{fig:multi-goal-distribution} shows the distribution of reached goals for \STAC/SAC for the agents in \Figref{fig:multigoal-trajectory}. Trajectories are collected from $20$ episodes of $20$ different agents trained with $20$ different seeds for each algorithm. We observe that with higher $\alpha$'s, more agent trajectories converge to the left two goals (G2 and G3), which is not the case for SAC and SQL. This shows that \STAC\ learns a more optimal solution to the MaxEnt objective in Eq.\eqref{eq:max_entr_obj}.
\begin{figure}[H]
\centering
\begin{subfigure}{.24\linewidth}
    \includegraphics[width=\linewidth]{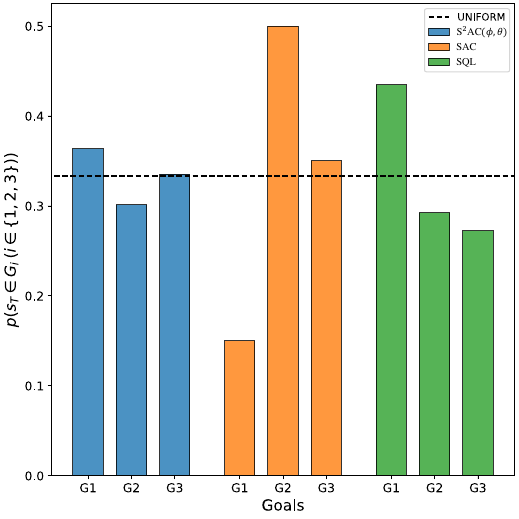}
    \caption{$\alpha=0.2$}
    \label{fig:multi-goal_distribution_a-0.2}
\end{subfigure}
\begin{subfigure}{0.24\linewidth}
    \includegraphics[width=\linewidth]{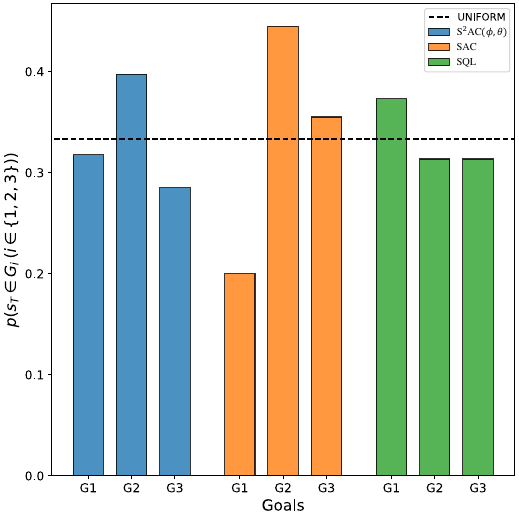}
    \caption{$\alpha=1$}
    \label{fig:multi-goal_distribution_a-1}
\end{subfigure}
\begin{subfigure}{0.24\linewidth}
    \includegraphics[width=\linewidth]{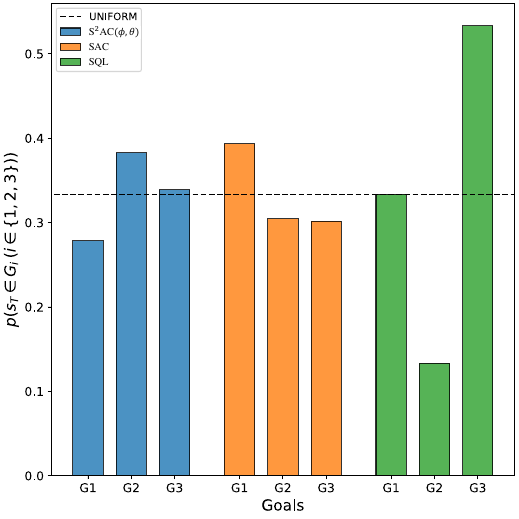}
    \caption{$\alpha=10$}
    \label{fig:multi-goal_distribution_a-10}
\end{subfigure}
\begin{subfigure}{.24\linewidth}
    \includegraphics[width=\linewidth]{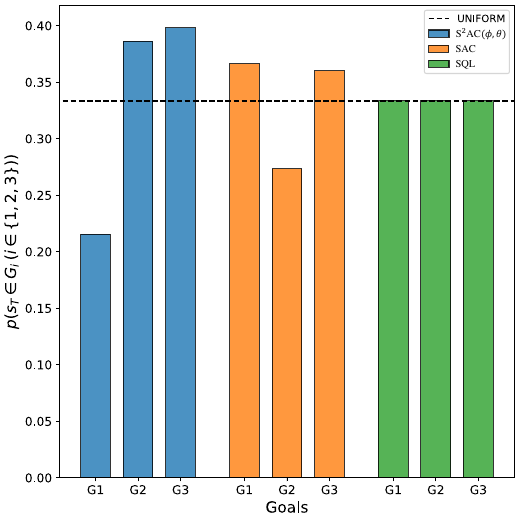}
    \caption{$\alpha=20$}
    \label{fig:multi-goal_distribution_a-20}
\end{subfigure}
    \caption{Distribution of reached goals for \STAC, SAC and SQL with different $\alpha$'s. The x-axis denotes different goals. The y-axis represents the ratio of trajectories that reach the goal.}
    \label{fig:multi-goal-distribution}
\end{figure}

\textbf{Entropy heatmap of \STAC\ in the multi-goal environment.}
\Figref{fig:entropy_heatmap} shows the entropy heatmap of \STAC\ with different $\alpha$'s. A brighter color corresponds to higher entropy. For \STAC, the higher $\alpha$, the higher the entropy on the left quadrant compared to the right one, i.e., the more contrast between the left and the right quadrants. For instance, In Figure~\ref{fig:multi-goal_stac_a-20_HM} (\STAC, $\alpha=20$), notice a clear green/yellow patch spanning the left side, while the right side is mostly dark blue except for the edges.
\begin{figure}[H]
\centering
\begin{subfigure}{.22\linewidth}
    \includegraphics[width=\linewidth]{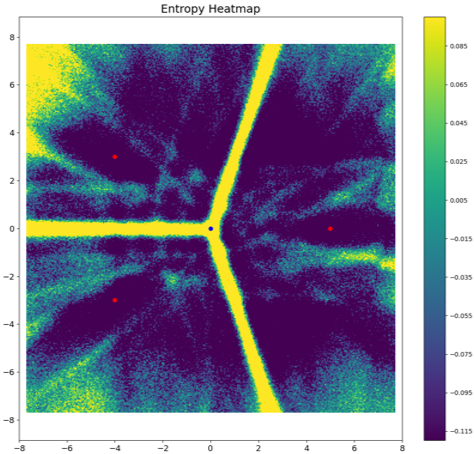}
    \caption{\STAC, $\alpha=0.2$}
    \label{fig:multi-goal_stac_a-0.2_HM}
\end{subfigure}
\begin{subfigure}{0.22\linewidth}
    \includegraphics[width=\linewidth]{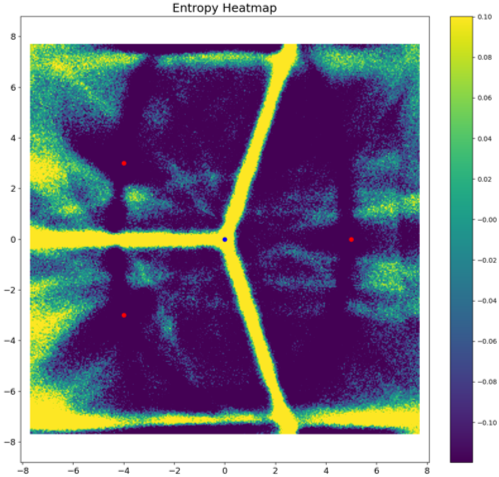}
    \caption{\STAC, $\alpha=1$}
    \label{fig:multi-goal_stac_a-1_HM}
\end{subfigure}
\begin{subfigure}{0.22\linewidth}
    \includegraphics[width=\linewidth]{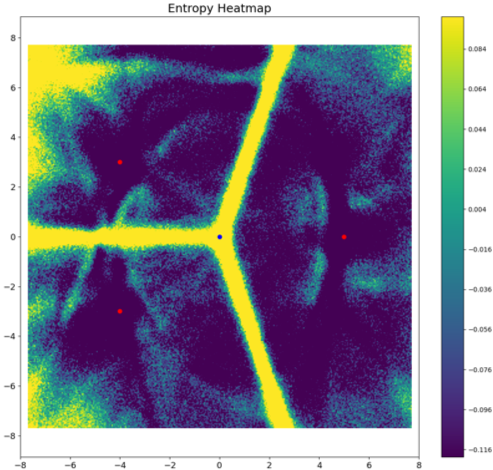}
    \caption{\STAC, $\alpha=10$}
    \label{fig:multi-goal_stac_a-10_HM}
\end{subfigure} 
\begin{subfigure}{.22\linewidth}
    \includegraphics[width=\linewidth]{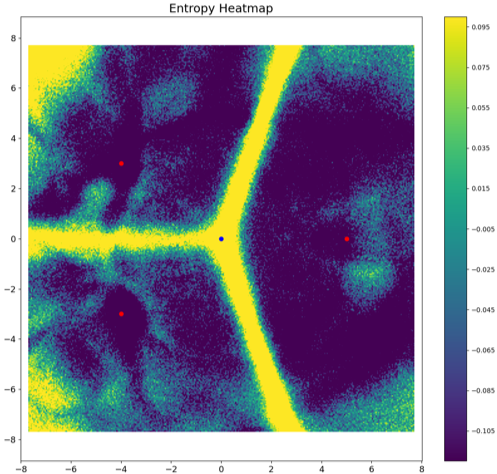}
    \caption{\STAC, $\alpha=20$}
    \label{fig:multi-goal_stac_a-20_HM}
\end{subfigure}\\
\caption{The entropy heatmap of \STAC\ in the multi-goal environment for different $\alpha$}
\label{fig:entropy_heatmap}
\end{figure}

\textbf{Smoothness of the Q-landscapes.} To assess the effect of the entropy, we visualize the Q-landscapes corresponding to six typical states $s\!\in\!\{s_o, s_a, s_b, s_c, s_d, s_e\}$ (marked in blue on the upper left of \Figref{fig:entropy_states}) across different trajectories to the goal and report their associated entropy $\cH(\cdot|s)$ (bottom left of \Figref{fig:entropy_states}). The blue dots correspond to $10$ SVGD particles at convergence. We observe that the Q-landscape becomes smoother with increasing $\alpha$. For instance, notice how the modes for state $s_c$ become more connected. Quantified measurements of smoothness are in \Figref{fig:smoothness}. We use two metrics $M_1$ and $M_2$ to measure the smoothness of the learned Q-landscape: (1) $M_1$: the average over the L1-norm of the gradient of the Q-value with respect to the actions across trajectories, \ie $\mathbb{E}_{\tau\sim\pi(a|s)}\Big[\mathbb{E}_{(s_t,a_t)\in \tau} \big[\frac{||\nabla_{a_t}Q(s_t,a_t)||_1}{d} \big] \Big]$. (2) $M_2$: The average over the L1-norm of the Hessian of the Q-value with respect to the actions across trajectories, \ie $\mathbb{E}_{\tau\sim\pi(a|s)}\Big[\mathbb{E}_{(s_t,a_t)\in \tau} \big[\frac{1}{d^2}\sum_{i,j} |\nabla^2_{a_t}Q(s_t,a_t)|_{i,j} \big] \Big]$. \Figref{fig:smoothness} shows that increasing $\alpha$ leads to consistently smaller gradients (\Figref{fig:smoothness_grad}) and less curvature (\Figref{fig:smoothness_hessian}). Hence, the entropy results in a smoother landscape that helps the sampling convergence.

\begin{figure}[H]
\centering
\includegraphics[width=0.9\linewidth]{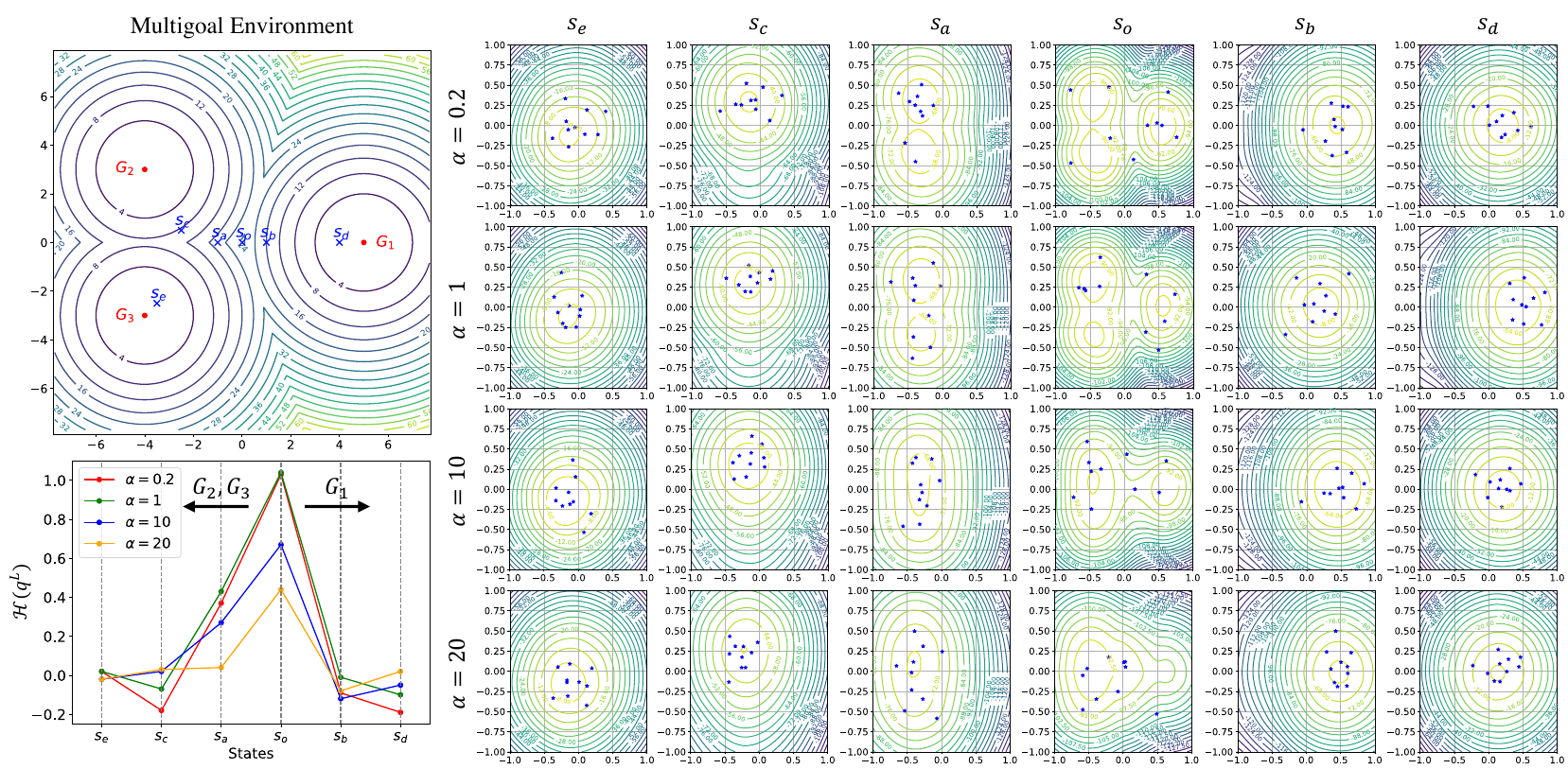}
\caption{
Results on the multi-goal environment. Increasing $\alpha$ yields smoother landscapes (\eg $s_o$). Notice how the modes become more connected (\eg for $s=s_a$ with increasing $\alpha$). The entropy at the different states is reported in the lower left figure. }
\label{fig:entropy_states}
\vspace{-2mm}
\end{figure}

\begin{figure}[H]
\centering
\begin{subfigure}{0.4\linewidth}
    \centering
    \includegraphics[width=0.6\linewidth]{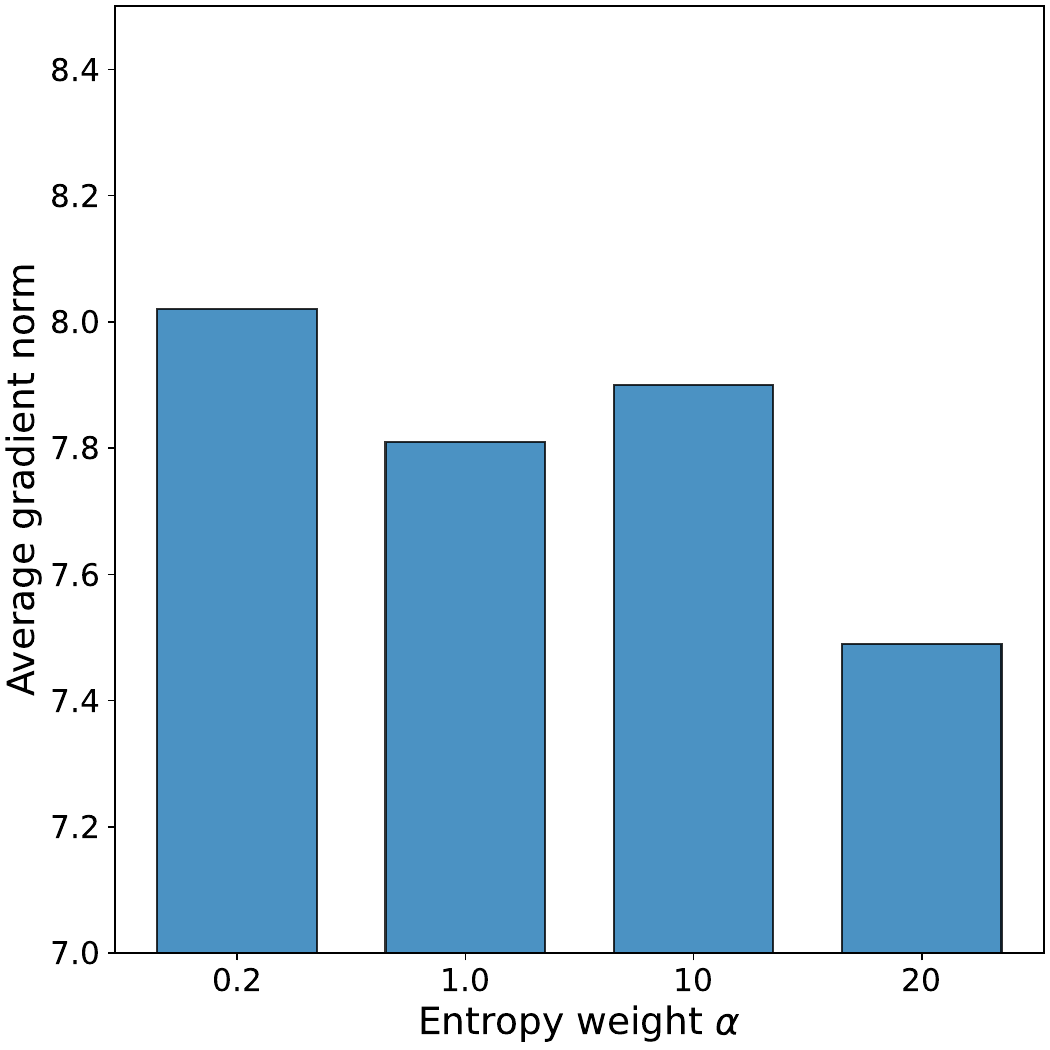}
    \caption{Average gradient across trajectories}
    \label{fig:smoothness_grad}
\end{subfigure} \hspace{2mm}
\begin{subfigure}{.4\linewidth}
    \centering
    \includegraphics[width=0.6\linewidth]{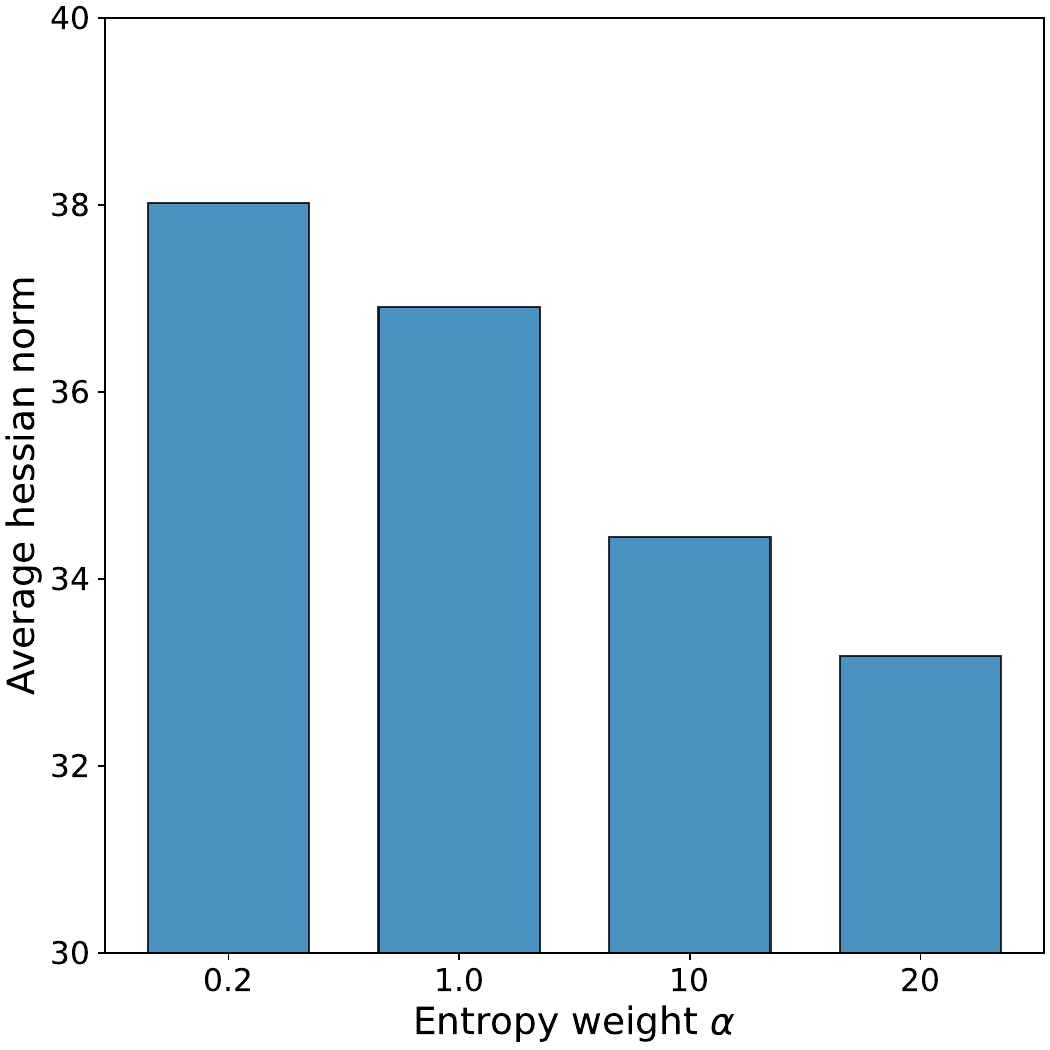}
    \caption{Average Hessian across trajectories}
    \label{fig:smoothness_hessian}
\end{subfigure}
\caption{Quantitative evaluation of the smoothness the Q-landscape of \STAC\ for different $\alpha$'s.}
\label{fig:smoothness}
\end{figure}

\textbf{Parametrization of $q^{0}$}. In \Figref{fig:mean_std_curve_initial_actor}, we visualizes the coordinates of the mean $\mu_\theta(s)$ and standard deviation $\sigma_\theta(s)$ of $q^{0}_{\theta}$ at different states $s \in \{s_o, s_e, s_b \}$ in the multigoal environment. As training goes on, $\mu_\theta(s)$ shifts closer to the nearest goals. For example, $\mu_\theta(s_b)$ becomes more positive during the training as it is shifting to $G_1$. Additionally, the model learns a high variance $\sigma_{\theta}(s)$ for the multimodal state $s_o$ and becomes more deterministic for the unimodal ones (\eg $s_e$ and $s_b$). As a result, in \Figref{fig:multigoal-robustness}, we observe that \STAC($\phi,\theta$) requires a smaller number of steps to convergence than \STAC($\phi$).

\textbf{Entropy estimation.} \Figref{fig:entropy_states} shows that the entropy is higher for states on the left side due to the presence of two goals, as opposed to a single goal on the right side (\eg $\cH(\pi_{\theta}(\cdot|s_a))\!<\!\cH(\pi_{\theta}(\cdot|s_o))$). Also, the entropy decreases when approaching the goals (\eg $\cH(\pi_{\theta}(\cdot|s_d))\!<\!\cH(\pi_{\theta}(\cdot|s_b))\!<\!\cH(\pi_{\theta}(\cdot|s_o))$). The same is valid along the paths to goal $G_1$. 

\begin{figure} [htb]
    \centering
     \includegraphics[width=0.35\linewidth]{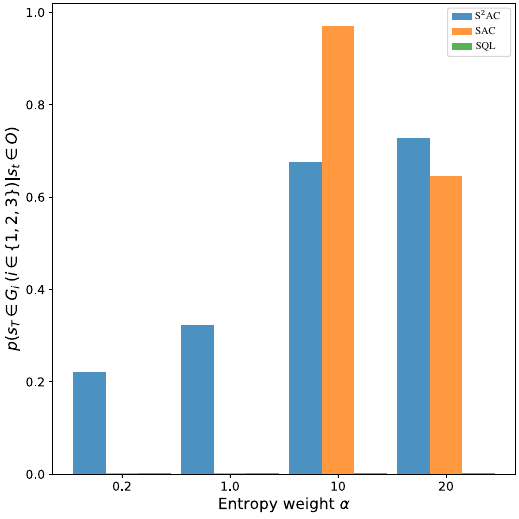}
      \caption{Distribution of reached goals after hitting an obstacle for \STAC, SAC and SQL.}
      \label{fig:obstacles}
\end{figure}

\textbf{Robustness/Adaptability.} In \Figref{fig:obstacles}, we report the distribution of reached goals after hitting an obstacle for \STAC, SAC and SQL for different $\alpha$'s. Notice that \STAC\ robustness, measured by the probability of reaching the goal for \STAC\ is consistently increasing with increasing $\alpha$. Intuitively, exploration is better with large values of $\alpha$, leading to better learning of the Q-landscape. In other words, from a given state, the agent is more likely to have explored more sub-optimal ways to reach the goal. So, when the optimal path is blocked with the barrier, the agents trained with \STAC\ are more likely to have learned several other ways to go around it. This is different from SAC, when the policy is uni-modal (Gaussian) and the agents are only able to escape the barrier and get to the goal for large $\alpha$'s ($\alpha \in {10,20}$). However, robustness in the case of SAC trained with large $\alpha$'s come at the expense of performance, \ie increased number of steps (See row 3 in Figure~\ref{fig:multigoal-robustness}). Besides, note that the number of SAC agents reaching the goals for $\alpha=20$ is less than the one for $\alpha=10$. This is due to the fact that higher $\alpha$'s lead to higher stochasticity and less structured exploration (the standard deviation of the Gaussian becomes very large). SQL fails to reach the goals once the obstacle is added. This shows that the implicit entropy in SQL is not as efficient as the explicit entropy in SAC and \STAC.

\textbf{Amortized \STAC.} In \Figref{fig:amortized_multigoal}, we report results of the amortized version of \STAC, \ie \STAC($\phi,\theta,\psi$) on the multigoal environment. Performance and robustness are comparable with the non-amortized version \STAC($\phi,\theta$) while having a faster inference (feedforward pass through $f_{\psi}(s,z)$). 

\begin{figure}[H]
\centering
    \centering
    \begin{subfigure}{1\linewidth}
    \centering
    \includegraphics[width=1\linewidth]{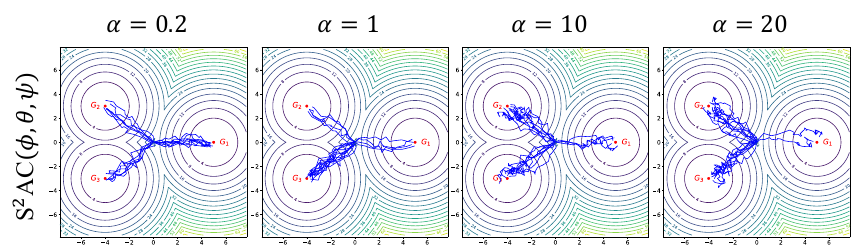}
    \caption{Performance of \STAC($\phi,\theta, \psi$) on the Multigoal environment}
\end{subfigure} \hspace{2mm}
\begin{subfigure}{1\linewidth}
    \centering
    \includegraphics[width=1\linewidth]{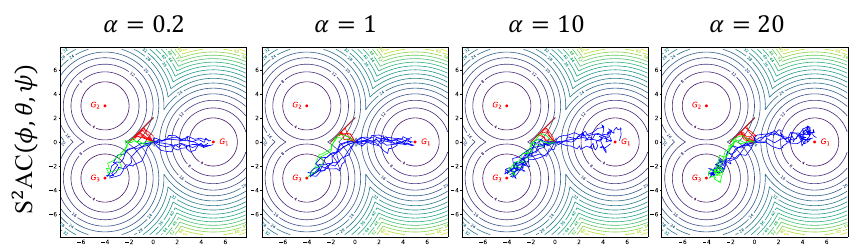}
    \caption{Performance of \STAC($\phi,\theta, \psi$) on the Multigoal environment with obstacles}
\end{subfigure}
    \caption{Performance of Amortized \STAC\ on the Multigoal environment}
    \label{fig:amortized_multigoal}
\end{figure}

\newpage



\section{Additional Results: MuJoCo}\label{appendix:mujoco}
\tabref{tab:mujoco_hyperparams} lists the \STAC\ hyper-parameters used in our experiments. Additionally, we give details on accelerating \STAC.

\begin{table}[h!]
\caption{Hyperparameters}
\centering
\setlength\tabcolsep{7pt}
\begin{tabular}{c|c|c}
\toprule
& Hyperparameter & Value   \\\midrule
\multirow{3}{*}{Training} & Optimizer & Adam  \\
& Learning rate  & $3 \cdot 10^{-4}$ \\
& Batch size & $100$ \\
\midrule
\multirow{4}{*}{Deepnet}
& Number of hidden layers (all networks) & $2$ \\
& Number of hidden units per layer & $256$ \\
& Number of samples per minibatch & $256$ \\
& Nonlinearity & ReLU \\
\midrule
\multirow{5}{*}{RL} & Target smoothing coefficient & $0.005$ \\
& Discount $\gamma$  & $0.99$ \\
& Target update interval& $1$ \\
& Entropy weight $\alpha$ &  1.0 for all environments, 0.2 Ant \\ 
& Replay buffer size $|\cD|$  & $10^{6}$ \\
\midrule
\multirow{5}{*}{SVGD} & initial distribution $q_0$& $\cN(\bm{0}, 0.5\bm{I})$ \\
& Learning rate $\epsilon$ &  $0.1$ \\
& Number of steps $L$ (\STAC($\phi$)) & $20$ \\
& Number of steps $L$ (\STAC($\phi,\theta$))& $3$ \\
& Number of particles $m$ & $10$\\ 
& Particles range (num. std) $t$ & 3\\ 
& Kernel variance & $\sigma = \frac{ \sum_{i,j} \| a_i - a_j \|^2  }{ 4(2 \log{m+1}) }$\\
\bottomrule
\end{tabular}
\label{tab:mujoco_hyperparams}
\end{table}
\textbf{Computational Efficiency.} 
Compared to SAC, running SVGD for $L$ steps requires $L$ additional back-propagation passes through the Q-network and a factor of $m$ (number of particles) increase in the memory complexity. In order to improve the efficiency of \STAC, we limit the number of particles $m$ to 10/20 and the number of SVGD steps $L$ to 10/20.

Additionally, we experiment with the following amortized version of \STAC. Specifically, we train a deepnet $f_{\psi}(s,z)$ to mimic the SVGD dynamics during testing, where $z$ is a random vector that allows mapping the same state to different particles. Note that we cannot use this deepnet during training as we need to estimate the closed-form entropy which depends on the SVGD dynamics. One way to train $f_{\psi}(s,z)$ is to run SVGD to convergence and train $f_{\psi}(s,z)$ to fit SVGD outputs. This however requires collecting a large training set of state action pairs by repeatedly deploying the policy. This might be slow and result in low coverage of the states that are rarely visited by the learned policy and hence result in poor robustness in case of test time perturbations. We instead propose an incremental approach in which $\psi$ is iteratively adjusted so that the network output $a = f_{\psi}(s,z)$ changes along the Stein variational gradient direction that decreases the KL divergence between the policy and the EBM distribution, \ie
\begin{equation}
\Delta f_{\psi}(z,s) = \frac{1}{m} \sum_{i=1}^{m} k(a_{i}, f_{\psi}(s,z)) \nabla_{a_i} Q(s,a_i) + \alpha \nabla_{a_i}k(a_i,f_{\psi}(s,z))  
\end{equation}
Note that $\Delta f_{\psi}$ is the optimal direction in the reproducing kernel Hilbert space, and is thus not strictly the gradient of Eq.\eqref{eq:sac_pi}, but it still serves a good approximation, \ie $\frac{\partial J}{\partial a_t} \propto \Delta f_{\psi}$, as explained by \cite{wang2016learning}. Thus, we can use the chain rule and backpropagate the Stein variational gradient into the policy network according to
\begin{equation}
\frac{\partial J(s)}{\partial \psi} \propto \E_{z} \left[ \Delta f_{\psi}(s,z)  \frac{\partial f_{\psi}(z,s)}{\partial \psi}  \right]. 
\label{eq:amortized}
\end{equation}
to learn the optimal sampling network parameters $\psi^{*}$. Note that the amortized network takes advantage of a Q-value that estimates the expected future entropy which we compute via unrolling the SVGD steps using Eq~\eqref{eq:svgd_entr_formula}.

The modified \STAC\ algorithm is described in Algorithm~\ref{alg:stac_2}. 

\begin{algorithm}[H]
    \caption{Stein Soft Actor Critic (\STAC) with Amortized policy (test-time)}
    \begin{algorithmic}[1]
    \STATE Initialize parameters $\phi$, $\theta$, $\psi$, hyperparameter $\alpha$, and replay buffer $\mathcal{D} \leftarrow \emptyset$ 
    \FOR{each iteration}
    \FOR{each environment step $t$}
    \STATE Sample action particles $\{a\}$ from $\pi_{\theta}(\cdot|s_t)$
    \STATE Select $a_t \in \{a\}$ using  exploration strategy
    \STATE Sample next state $s_{t+1} \sim p(s_{t+1}|s_{t},a_{t})$
    \STATE Update replay buffer $\mathcal{D} \leftarrow \mathcal{D} \cup (s_t,a_t,r_t,s_{t+1})$
    \ENDFOR
    \FOR{each gradient step}
    \STATE \textbf{Critic update:}
    \STATE \hspace{4mm} Sample particles $\{a\}$ from an EMB sampler $\pi_{\theta}(\cdot|s_{t+1})$
    \STATE \hspace{4mm} Compute entropy $\mathcal{H}(\pi_{\theta}(\cdot|s_{t+1}))$ using Eq.\eqref{eq:closed_form_entropy} 
    \STATE \hspace{4mm} Update $\phi$ using Eq.\eqref{eq:critic_update}
    \STATE \textbf{Actor update:}
    \STATE \hspace{4mm} Update $\theta$ using Eq.\eqref{eq:actor_loss}
    \STATE \hspace{4mm} Update $\psi$ using Eq.\eqref{eq:amortized}
    \ENDFOR
    \ENDFOR
    \end{algorithmic}
    \label{alg:stac_2}
    \end{algorithm}



\paragraph{Evaluation with the Rliable Library.}
Performances curves in \Figref{fig:mujoco_3env} are averaged over 5 random seeds and then smoothed using Savitzky-Golay filtering with window size 10. Additionally, we report metrics from the Rliable Library~\citep{agarwal2021deep} in Fig.~\ref{fig:mujoco_3env}, including 

\begin{itemize}

\item \textbf{Median}: Confidence interval of the median performance of each algorithm across different seeds, averaged over different MuJoCo environments.

\item \textbf{Mean}: Confidence interval of the average performance of each algorithm across different seeds and environments. 

\item \textbf{IQM (Interquantile means)}: Instead of computing the average performance on all trials, IQM shows the mean of the middle 50 percent of performance across different seeds.

\item \textbf{Optimality Gap}: The area between results curve of baseline algorithms and the horizontal line at the average performance of \STAC\ $(\phi,\theta)$.

\item \textbf{Probability of improvement over baselines}: The average probability that \STAC\ $(\phi,\theta)$ can make performance improvements over baseline algorithms.

\end{itemize}

The parameterized version of \STAC\ has the best performance among baselines in all the considered metrics. It has a probability of $\sim$65\% in outperforming  SAC-NF and $\sim$80\% in outperforming IAF.


\end{document}